\documentclass{article}

\usepackage[preprint]{neurips_2026} % Preprint version (arXiv etc.)

\usepackage[utf8]{inputenc} % allow utf-8 input
\usepackage[T1]{fontenc}    % use 8-bit T1 fonts
\usepackage{hyperref}       % hyperlinks
\usepackage{url}            % simple URL typesetting
\usepackage{booktabs}       % professional-quality tables
\usepackage{amsfonts}       % blackboard math symbols
\usepackage{nicefrac}       % compact symbols for 1/2, etc.
\usepackage{microtype}      % microtypography
\usepackage{xcolor}         % colours

\usepackage{graphicx}
\usepackage{subcaption}
\usepackage{wrapfig}        % text wrapping around small figures
\usepackage{bm} % for bold math symbols
\usepackage{multirow}
\usepackage{amsmath}
\usepackage{amssymb}
\usepackage{mathtools}
\usepackage{amsthm}
\usepackage{enumitem}
\usepackage{algorithm}
\usepackage{algpseudocode}

\makeatletter
\providecommand{\theHALG@line}{}
\renewcommand{\theHALG@line}{\thealgorithm.\arabic{ALG@line}}
\makeatother

\usepackage[capitalise,noabbrev]{cleveref}
\crefname{appendix}{appendix}{appendices}
\Crefname{appendix}{Appendix}{Appendices}

\usepackage{aliascnt}
\newcommand{\sharedtheorem}[2]{%
  \newaliascnt{#1}{theorem}%
  \newtheorem{#1}[#1]{#2}%
  \aliascntresetthe{#1}%
}
\theoremstyle{plain}
\newtheorem{theorem}{Theorem}[section]
\sharedtheorem{corollary}{Corollary}
\theoremstyle{definition}
\sharedtheorem{assumption}{Assumption}
\theoremstyle{remark}
\sharedtheorem{remark}{Remark}
\crefname{assumption}{assumption}{assumptions}
\Crefname{assumption}{Assumption}{Assumptions}
\crefname{corollary}{corollary}{corollaries}
\Crefname{corollary}{Corollary}{Corollaries}
\crefname{remark}{remark}{remarks}
\Crefname{remark}{Remark}{Remarks}

\usepackage[textsize=tiny]{todonotes}

\newcommand{\samuon}{SAMuon}
\newcommand{\samuonlite}{SAMuon-lite}

\title{Spectral Allocation: Why Muon Outperforms Adam, and How to Improve Muon}

\author{%
  Xiaodong Wu\thanks{Corresponding author: \texttt{xw338@cam.ac.uk}} \\
  Department of Engineering\\
  University of Cambridge\\
  \texttt{xw338@cam.ac.uk} \\
  \And
  Wenyi Yu \\
  Department of Electronic Engineering\\
  Tsinghua University \\
  \And
  Chao Zhang \\
  Department of Electronic Engineering\\
  Tsinghua University \\
  \And
  Philip Woodland \\
  Department of Engineering\\
  University of Cambridge \\
}

\begin{document}

\maketitle

\begin{abstract}
Orthogonal optimisers such as Muon can substantially accelerate large language model pretraining relative to Adam, yet the mechanism remains incompletely understood. We investigate this through an \emph{out-of-sample spectral probing analysis} of Transformer loss landscapes. At checkpoints along real training trajectories, we decompose each momentum buffer into its singular directions and estimate the loss-optimal step size along each direction on held-out data. The resulting spectral profile is anisotropic yet stable across batches and training stages, and consistent across the optimisers and model scales: a \emph{volatile head} operating at the Edge-of-Stability supports a much smaller step size than the \emph{tolerant bulk}, which permits substantially larger steps. This profile provides a unified \emph{spectral allocation} account of why Muon outperforms Adam, which outperforms SGD. It also exposes a limitation of Muon's uniform scaling: it still underutilises the bulk. Guided by this finding, we introduce \textbf{Spectral-Aware Muon (SAMuon)}, which holds the head at the Muon scale and amplifies the bulk using a \emph{static spectral prior}. We provide two variants: the complete SAMuon follows the measured profile using a low-rank randomised SVD and the simplified SAMuon-lite uses a two-level approximation via rank-one power iteration. Neither method adds persistent optimiser state or notable extra FLOPs beyond Muon at scale, and the idealised exact-whitening versions of both retain Muon's asymptotic convergence rate under standard assumptions. Across ``modded-nanogpt'' models from 124M to 1B parameters, both variants outperform tuned AdamW and Muon (Scion implementation) baselines in all evaluated model-scale and batch-size configurations. SAMuon requires $13.3\%$--$24.0\%$ fewer training tokens to reach the same validation loss as Muon, while \samuonlite{} retains most of this gain with near-zero wall-clock overhead.

\end{abstract}

\section{Introduction}
\label{sec:intro}
The pretraining of Transformer-based decoder-only Large Language Models (LLMs) \citep{GPT-2} is constrained by the efficiency of optimisers like Adam \citep{adam}. Recently, orthogonal optimisers such as Muon \citep{muon, moonlight} have emerged as a promising research direction. By whitening the (momentum) gradient matrix to generate updates, these methods achieve clear convergence acceleration compared to Adam, particularly during the transient phase (the early, non-asymptotic stage) of pretraining \citep{FantasticOpt}.

The empirical success of orthogonal optimisers has prompted a number of theoretical explanations, yet the true source of their fast convergence remains an open question. The dominant view frames Muon as spectral-norm steepest descent \citep{OldOptNewNorm, MuonSpectralNorm, ModularDuality, MuonTrustRegion}, which accounts for its stability and learning-rate transfer but, as \citet{polargrad} note, does not by itself explain the observed convergence advantage on highly curved pretraining objectives, since whitening ``effectively removes all curvature information''. A more recent line of explanation turns directly to curvature: \citet{isotropic} showed that orthogonalisation is loss-optimal under an \emph{isotropic} curvature model, and concurrent work by \citet{MuonCurvature} attributes the advantage of Muon over Adam and gradient descent to Muon's lower directional-sharpness penalty. These insights are real progress. However, the isotropic curvature model \citep{isotropic} is, as our measurements show, violated in practice. Meanwhile, the analysis in \citet{MuonCurvature} relies on an aggregate, \emph{in-sample} Hessian diagnostic (computed only on the gradient batch) which compresses the update's entire spectral geometry into a single scalar and consequently misses the crucial curvature information hidden in the full spectrum. In this paper, we directly probe the behaviour of the full spectrum of the momentum buffer on held-out batches, the exact behaviour we aim to improve in practical training scenarios. This allows us to ask the key question that previous work leaves open: \textbf{what is the curvature behaviour along each spectral direction of the momentum buffer in real Transformer training, and what does this imply for optimiser design?} Answering it gives a \emph{spectral lens} on the loss landscape and a unified view for understanding and improving optimisers.

Our paper is based on an \emph{out-of-sample} spectral probing analysis. Probing the singular basis of the momentum buffer that orthogonal optimisers whiten, we measure the optimal step size that the loss landscape of held-out batches allows along each singular direction. The resulting loss-optimal step profile is highly anisotropic and stable across different batches, training stages and model scales. The profile is highly forgiving across the \emph{tolerant bulk} and drops sharply near a single \emph{volatile head}, with the decline being close to linear in log rank over the leading ranks. Viewed in this way, the major optimiser families fall on a single axis: momentum-based stochastic gradient descent (SGD) allocates step size in proportion to each singular value, exactly opposite to what the landscape suggests; the coordinate-wise rescaling of Adam dampens but cannot overcome this misallocation; and Muon, by whitening to a uniform scale, reallocates step mass from the head into the more tolerant bulk. This yields, to our knowledge, the first unified \textbf{Spectral Allocation} account of why Muon outperforms Adam, which in turn outperforms SGD in Transformer pretraining.

In addition, the same view suggests that there is still room for improvement in the uniform scaling of Muon. The dominant singular direction, termed the \textbf{Volatile Head}, carries the majority of the gradient signal yet allows the smallest step size. It sits at the ``Edge-of-Stability'' (EoS) \citep{EoS} and its optimal step size coincides with the set global learning rate. Meanwhile, the remaining \textbf{Tolerant Bulk} allows step sizes several times larger and is mostly flat beyond the leading few ranks. Therefore, the uniform spectral allocation of Muon realises only a small fraction of the per-iteration loss reduction that an idealised, per-rank-optimal spectral allocation would permit (the right panel of \cref{fig:lr_analysis}). This loss reduction could likely be recovered by amplifying the bulk while holding the head at the Muon scale.

To make use of this potential improvement without incurring significant computational overhead, we propose \textbf{Spectral-Aware Muon (SAMuon)}. Rather than learning the profile online like SOAP, which applies Adam in Shampoo's eigenbasis \citep{SOAP}, \samuon{} imposes a \emph{static}, head-anchored prior on the whitened update. Two algorithms \emph{hold} the volatile head at the Muon scale that limits the step and increase the scale of the rest of the spectrum. Two variants of the \samuon{} algorithm are provided: the complete \samuon{} algorithm employs a profile that closely follows the empirically measured log-rank-linear shape through a rank-$k$ randomised singular value decomposition (SVD) of the momentum buffer at each update step; the simplified \samuonlite{} algorithm keeps a deliberately coarse two-level profile which amplifies the whole bulk uniformly to scale $\gamma$ while pinning only the head using power iteration at each step. Both variants add no persistent optimiser state beyond the single buffer of Muon and negligible extra floating-point operations (FLOPs). Note that both methods introduce only a single tuned scale hyperparameter $\gamma$; setting $\gamma = 1$ recovers Muon for both variants.

We validate both \samuon{} variants on ``modded-nanogpt'' \citep{moddednanogpt} pretraining following the protocol in \citet{scion}, across models from 124M to 1B parameters and batch sizes from 1024 to 4096. Hyperparameters are tuned per batch size at 124M and transferred across model scales. The 124M budget is ${\sim}4\times$ the Chinchilla-optimal allocation \citep{chinchilla}, placing training in the over-trained regime where optimiser gains are hardest to retain \citep{FantasticOpt}. By better exploiting the tolerant bulk, the complete \samuon{} improves on AdamW and the Muon (Scion implementation) baseline in each trained model and batch size combination, reaching the same final validation loss as Muon with $13.3\%$--$24.0\%$ fewer tokens. \samuonlite{} recovers most of this improvement at nearly the same wall-clock speed as Muon due to the more efficient implementation.

\textbf{Our contributions:}

\textbf{1) Out-of-sample spectral probing and a unified account of optimiser performance.} To answer the question of why the uniform whitening of Muon works so well with the loss landscape in Transformer pretraining, we introduce an out-of-sample spectral probing framework which measures the optimal step size that the Transformer loss landscape tolerates along each singular direction of the momentum buffer. It identifies a training-stable structure: a profile that is mostly flat across the bulk and drops sharply near a single volatile head. Viewed as a spectral allocation, this gives a unified account of why the performance of Muon exceeds that of Adam, which in turn exceeds that of SGD. It also provides insights into how to improve Muon.

\textbf{2) The SAMuon optimisers.} Based on the findings of the spectral probing analysis, we propose two variants of Spectral-Aware Muon (SAMuon), which distil this measured structure into two static, head-anchored priors on the spectrum of the momentum buffer. \samuon{} follows the measured log-rank-linear shape via a rank-$k$ randomised SVD at each update step, and \samuonlite{} keeps a two-level simplification via power iteration at each step, with nearly the same wall-clock speed as Muon. Neither adds persistent optimiser state over Muon. Both add only one additional tunable hyperparameter and negligible extra FLOPs.

\textbf{3) Empirical and theoretical validation.} On the ``modded-nanogpt'' pretraining benchmark, using models from 124M to 1B parameters and matched hyperparameter tuning for all methods, both proposed variants improve over AdamW and the strong Muon (Scion \citep{scion}) baseline in each trained model and batch size combination. \samuon{} shows a $13.3\%$--$24.0\%$ token-efficiency improvement over Muon, and \samuonlite{} matches or nearly matches this performance with a token-efficiency improvement of $13.3\%$--$22.1\%$. A convergence guarantee covering the idealised exact-whitening versions of both variants is also provided.

\section{Related work}
\label{sec:related}

\textbf{Theoretical Interpretations of Muon's Success.}
The prevailing view of Muon is norm-geometric, \textit{i.e.} orthogonalisation is the steepest-descent step under the spectral norm. This view is developed along three threads: as a norm/duality derivation of the update \citep{OldOptNewNorm, ModularDuality, muon, large2024modular, yang2023spectral, MuonSpectralNorm}, as a linear-minimisation-oracle or non-Euclidean trust-region step with attendant convergence guarantees \citep{scion, MuonTrustRegion, lihong2025convergence} and as an implicit bias towards spectral-norm max-margin solutions \citep{fan2025implicit}. Collectively, these works establish that orthogonalisation is the steepest-descent step for the linearised loss under a constraint on the update's spectral norm. Such spectral-norm control yields width- and depth-robust update scales and learning-rate transfer, while ensuring that the resulting iterates carry a characterisable spectral-norm max-margin bias. What they do not address is the curvature of the landscape \emph{inside} the norm ball, which is the most direct link to the training acceleration. \citet{isotropic} makes this missing premise explicit, showing that gradient orthogonalisation is loss-optimal under an \emph{isotropic} curvature model in which all post-whitening directions are equally easy to descend, an assumption our measurements show is violated in real Transformer training. Most recently, and arguably closest to ours, the concurrent work of \citet{MuonCurvature} shows that the advantage of Muon over Adam and gradient descent stems from a smaller normalised directional-sharpness penalty, achieved by balancing update energy across curvature groups, rather than from spectral-norm properties; they further verify that the \emph{in-sample} Hessian eigenbasis aligns closely with the gradient's singular basis, directly supporting the per-basis construction of our spectral probes. Our spectral probing analysis can be viewed as an extension of this work, measuring the full anisotropic profile at per-rank granularity along real Transformer training trajectories in an \emph{out-of-sample} manner. This measured out-of-sample spectral profile is more closely aligned with practical training scenarios (where convergence on unseen batches is concerned) and enables us to develop a low-cost, improved Muon optimiser (SAMuon), whose notable empirical improvement, in turn, verifies this account.

\textbf{Spectral Analysis of Deep Neural Networks.} Empirical studies of the Hessian eigenspectrum have long established that neural network loss landscapes are highly anisotropic, the spectrum splitting into a bulk concentrated near zero and a small number of outlier eigenvalues orders of magnitude larger \citep{EmpiricalHessian}. For Transformers, \citet{adamtransformer} analyse this structure at the granularity of parameter blocks, attributing the failure of single-learning-rate SGD, relative to coordinate-wise Adam, to this inter-block curvature disparity, \textit{i.e.}\ ``block heterogeneity''. Complementary to their work, we argue that the distribution of curvature \emph{within} each block across the singular directions of the update is another key aspect that separates orthogonal optimisers from Adam (\cref{sec:optimiser_implications}). There is also a body of work discussing the spectral anisotropy of the gradient itself \citep{SUMO,AniGrad1,lowrankMuonBad}, instead of the curvature behaviour that is our focus. Specifically, \citet{lowrankMuonBad} provides independent evidence for our central finding: truncating the bottom singular directions of the Muon update markedly degrades its performance, consistent with our finding that the advantage of Muon resides in better utilising the spectral bulk. 

\textbf{Adaptive Methods and Practical Muon Variants.} Beyond convergence speed, a central appeal of Muon \citep{muon} is its low memory overhead. Muon maintains only a single momentum buffer, rather than the two moments used by Adam \citep{adam} and the even heavier memory requirement of most second-order methods. Kronecker-Factored Approximate Curvature (K-FAC) \citep{K-FAC}, Shampoo \citep{Shampoo}, and SOAP \citep{SOAP} precondition with Kronecker-factored curvature statistics whose factors scale with model size, and some recent Muon variants reintroduce the same burden in pursuit of adaptivity: AdaMuon \citep{adamuon} adds an element-wise second moment, Newton--Muon \citep{newtonmuon} an explicit activation-covariance right preconditioner, and Mousse \citep{mousse} a full set of Shampoo Kronecker factors. Each improves on Muon but loses its memory advantage, carrying additional state that scales with the model. A complementary hybrid, COSMOS \citep{cosmos}, instead confines adaptive (SOAP-style) preconditioning to the leading eigen-subspace of the gradient and applies a low-cost orthogonalised update to the remainder. Its design echoes our central finding, but is more complex than our head-only static treatment. LiMuon \citep{limuon} likewise applies a randomised low-rank SVD to the momentum buffer, but towards memory compression of the update rather than shaping its spectral allocation, \textit{i.e.} the same computational primitive as \samuon{} but with a different purpose. Two lines of research instead keep the footprint close to the single buffer used by Muon. NorMuon \citep{normuon} adds only a per-neuron second-moment vector to rescale the orthogonalised update; its operation is complementary to ours and potentially integrates with our method rather than competing with it. Scion \citep{scion} is a renormalised Muon which also replaces the auxiliary Adam optimiser used for embeddings with a lightweight signed update. \samuon{} retains this single-buffer footprint while exploiting the within-block spectral structure through a fixed, measurement-derived shaping of the update spectrum, at negligible additional FLOPs and with a single \emph{tuned} extra hyperparameter ($\gamma$; the profile rank of \samuon{} follows a fixed width rule and is not tuned). Note that we build our work directly on Scion, which is our primary baseline and is referred to as Muon in \cref{sec:experiments}.

\section{Preliminaries}
\label{sec:preliminaries}
The Muon optimiser \citep{muon} operates independently on each 2D weight matrix $\mathbf{W}^{(t)} \in \mathbb{R}^{m \times n}$ of the target model, where we take $m \le n$ without loss of generality. At step $t$, given the mini-batch gradient $\mathbf{G}^{(t)} = \nabla_{\mathbf{W}} \mathcal{L}(\mathbf{W}^{(t)}; \mathcal{B}_t)$, Muon first integrates the gradient into a momentum buffer $\mathbf{M}^{(t)}$, and it is this buffer that is orthogonalised to produce the update:
\begin{align}
    \mathbf{M}^{(t)} &= \mu\, \mathbf{M}^{(t-1)} + (1-\mu)\, \mathbf{G}^{(t)}, \label{eqn:muon_momentum}\\
    \mathbf{O}^{(t)} &= \mathrm{NS}\big(\mathbf{M}^{(t)}\big) \;\approx\; \big(\mathbf{M}^{(t)} \mathbf{M}^{(t)\top}\big)^{-1/2}\, \mathbf{M}^{(t)} \;=\; \mathbf{U}\mathbf{V}^\top, \label{eqn:muon_whitening}\\
    \mathbf{W}^{(t+1)} &= \mathbf{W}^{(t)} - \eta\, \kappa\, \mathbf{O}^{(t)}, \label{eqn:muon_step}
\end{align}
where $\mu$ is the momentum coefficient, $\eta$ is the learning rate and $\kappa$ is a fixed optimiser-specific scaling factor. For the Scion implementation of Muon used in our paper, the fixed scale for a weight matrix is $\kappa = \sqrt{d_{\mathrm{out}}/d_{\mathrm{in}}}$, where $d_{\mathrm{out}}$ and $d_{\mathrm{in}}$ are the matrix's original output and input dimensions, respectively \citep{scion}. For simplicity, we also assume $\mathbf{M}^{(t)}$ to have full row rank (with rank $r = m$). Here and throughout the paper, we write the (reduced) SVD of any matrix as $\mathbf{A} = \mathbf{U}\mathbf{S}\mathbf{V}^\top = \sum_{i} \sigma_i\, \bm u_i \bm v_i^\top$, where $\mathbf{S}$ is the diagonal matrix of singular values $\sigma_1 \ge \sigma_2 \ge \dots \ge \sigma_m \ge 0$, $\bm u_i, \bm v_i$ are the corresponding left and right singular vectors, and the index $i$ is the \emph{spectral rank index} of the singular component $\sigma_i \bm u_i \bm v_i^\top$; in \cref{eqn:muon_whitening}, $\mathbf{U}$ and $\mathbf{V}$ are from this SVD of the buffer $\mathbf{M}^{(t)}$. The operator $\mathrm{NS}(\cdot)$ denotes the Newton-Schulz iteration, which approximates the polar factor $\mathbf{U}\mathbf{V}^\top$ of its input using only a small number of matrix products, avoiding an explicit SVD. The iteration-specific coefficients and finite-iteration response used in our experiments are reported in \cref{app:ns_parameters}. In spectral terms, exact whitening maps every non-zero singular value to one. Thus, under exact whitening, all $m$ singular components within each weight matrix receive the same coefficient $\eta\kappa$; the custom finite $4 \times 3$ scheme closely approximates this uniform allocation over the operative range. The analysis in this paper focuses on this uniform spectral scaling.

\section{Spectral probing for the loss-optimal step profile}
\label{sec:linesearch_analysis}

The Muon optimiser \citep{muon} achieves rapid convergence by applying a uniform scaling to the spectrum of the momentum buffer and removing gradient anisotropy. To find \textit{why} this isotropic whitening improves over methods such as SGD and Adam, and to understand its structural suboptimality, we must investigate the true underlying geometry of the loss landscape. To this end, we conduct a fine-grained spectral analysis of the Transformer loss landscapes, measuring the loss-optimal step across the target spectrum.

\subsection{Methodology: Spectral probing}
\label{subsec:probing_method}

To construct the optimal step profile of the loss landscape, we propose a \textit{Spectral Probing} procedure that measures the loss landscape along each direction in the spectral basis of the momentum buffer across the entire model in an offline manner, drawing on the directional-probing methodology of \citet{iEF}.

\textbf{Per-Rank Probe Construction.}
Let $\mathcal{L}(\bm{\theta}_t; \mathcal{B})$ be the loss on a batch $\mathcal{B}$ at the checkpoint parameter state $\bm{\theta}_t$, where $t$ denotes the training iteration of the checkpoint. We probe the spectrum of the updated momentum buffer for the $N_{\mathrm W}=72$ two-dimensional weight matrices in the model, indexed by $j$. For each matrix, we integrate the gradient $\mathbf{G}_j$ of a \emph{gradient batch} $\mathcal{B}_{\mathrm{g}}$ into the momentum buffer $\mathbf{M}_j$ stored at the checkpoint, $\mathbf{M}_j \leftarrow \mu\, \mathbf{M}_j + (1-\mu)\, \mathbf{G}_j$ (with $\mu$ the momentum coefficient of the underlying optimiser), exactly reproducing the buffer state the optimiser operates on when generating parameter updates in practical training scenarios. Crucially, the response on a \emph{disjoint out-of-sample batch} $\mathcal{B}_{\mathrm{o}} \neq \mathcal{B}_{\mathrm{g}}$ is measured, \textit{i.e.}\ the loss reduction an update step achieves on unseen data, instead of the gradient batch that was used to form the update. This protocol is thus denoted as \emph{out-of-sample} spectral probing. Note that an in-sample measurement used by curvature analyses in prior works reports a far more permissive spectral tail that largely reflects refitting the gradient batch rather than generalisation on unseen data (see \cref{fig:insample_vs_heldout} in \cref{app:insample_probing}). The SVD of the updated momentum buffer is $\mathbf{M}_j = \sum_{d} \sigma_{j,d}\, \bm u_{j,d} \bm v_{j,d}^\top$. The singular components are indexed by their \emph{rank} $d$, \textit{i.e.}\ the position of $\sigma_{j,d}$ in the descending spectrum, and grouped across weight matrices by rank. This grouping is well-defined for GPT-2-style architectures because every probed matrix shares the same minimum dimension, the model width $d_\text{model}$, so each matrix carries exactly $d_\text{model}$ singular components and the $d$-th component can be aggregated model-wide. Our analysis therefore focuses on intra-block spectral behaviour instead of inter-block heterogeneity. The \emph{rank-$d$ probe} $\Delta\bm\theta^{(d)}$ collects the $d$-th singular component $\mathbf{P}_{j,d}$ of every probed matrix to form a pseudo parameter update
\begin{equation}
\label{eqn:rank_probe}
    \mathbf{P}_{j,d} = -\,\kappa_j\,\bm u_{j,d} \bm v_{j,d}^\top \,\, ,\qquad \Delta\bm\theta^{(d)} = \operatorname{vec}\big(\{\mathbf{P}_{j,d}\}_{j=1}^{N_{\mathrm W}}\big) \quad (d = 1, \dots, d_\text{model}),
\end{equation}
where $\operatorname{vec}$ concatenates the entries of all $N_{\mathrm W}$ matrices into a single vector acting on the parameter state $\bm{\theta}_t$, and $\kappa_j$ is the fixed per-weight-matrix scaling of Scion for the $j$-th matrix. Note that $\mathbf{P}_{j,d}$ has the same orientation and scaling as the rank-$d$ component that the actual Muon (Scion) optimiser applies to the $j$-th momentum buffer, which makes the probe's loss-optimal step size directly comparable with the effective learning rate used in the Muon (Scion) optimiser's trajectory (the red dashed line in the middle panel of \cref{fig:lr_analysis}). Also, every matrix contributes exactly one rank-one component to each probe, which prevents the curvature measurements from being confounded by unequal numbers of aggregated components from different matrices.

\textbf{Measurement via Local-Quadratic Approximation.}
For each probe, we primarily measure the \emph{loss-optimal step size}, defined as the step size that minimises the \emph{out-of-sample} loss along the probe direction, $\eta^*_d = \arg\min_{\eta \ge 0} \mathcal{L}(\bm{\theta}_t + \eta\, \Delta\bm\theta^{(d)}; \mathcal{B}_{\mathrm{o}})$. Running an exact line search for all ranks of every checkpoint is prohibitively expensive. Therefore, we estimate $\eta^*_d$ from the local-quadratic approximation (LQA) of the loss along the probe direction,
\begin{equation}
\label{eqn:lqa}
    \mathcal{L}(\bm{\theta}_t + \eta\, \Delta\bm\theta; \mathcal{B}_{\mathrm{o}}) \;\approx\; \mathcal{L}(\bm{\theta}_t; \mathcal{B}_{\mathrm{o}}) \;+\; \eta\, \bm g^\top \Delta\bm\theta \;+\; \tfrac{1}{2}\, \eta^2\, \Delta\bm\theta^\top \mathbf{H}\, \Delta\bm\theta,
\end{equation}
where $\bm g$ and $\mathbf{H}$ denote the full-parameter gradient and Hessian of the out-of-sample batch loss $\mathcal{L}(\cdot\,;\mathcal{B}_{\mathrm{o}})$ at $\bm\theta_t$. Minimising \cref{eqn:lqa} over $\eta$ yields
\begin{equation}
\label{eqn:eta_d}
    \eta^*_d \;=\; -\,\frac{\bm g^\top \Delta\bm\theta^{(d)}}{\Delta\bm\theta^{(d)\top}\, \mathbf{H}\, \Delta\bm\theta^{(d)}}.
\end{equation}
The directional curvature in the denominator does not require an explicit Hessian, since rearranging \cref{eqn:lqa} at a small displacement $\varepsilon$ gives the finite-difference estimate
\begin{equation}
\label{eqn:fd_curv}
    \Delta\bm\theta^\top \mathbf{H}\, \Delta\bm\theta \;\approx\; \frac{2}{\varepsilon^2}\Big[\,\mathcal{L}(\bm{\theta}_t + \varepsilon\, \Delta\bm\theta; \mathcal{B}_{\mathrm{o}}) \;-\; \mathcal{L}(\bm{\theta}_t; \mathcal{B}_{\mathrm{o}}) \;-\; \varepsilon\, \bm g^\top \Delta\bm\theta\,\Big],
\end{equation}
which reuses the pre-computed loss and gradient and consumes only \emph{one additional forward pass} per probe. The resulting collection $\{\eta^*_d\}_{d=1}^{d_\text{model}}$ constitutes the \emph{out-of-sample loss-optimal spectral step profile} of the loss landscape (the \emph{measured profile} for short) at checkpoint $\bm\theta_t$, which \cref{fig:lr_analysis} tracks across training iterations. Interpreting this profile relies on local quadraticity and, for a positive interior optimum, on the probe being a descent direction on the held-out batch and having positive directional curvature. These assumptions are stated in \cref{app:spectral_probing_assumption}, which also validates the local quadraticity of the Transformer pretraining loss landscape and reports the filtering of probes that fail either the descent or curvature check. Such probes account for only $0.52\%$ of the sample and are excluded from the final spectral profile.

\begin{figure*}[ht]
    \centering
    \includegraphics[width=\textwidth]{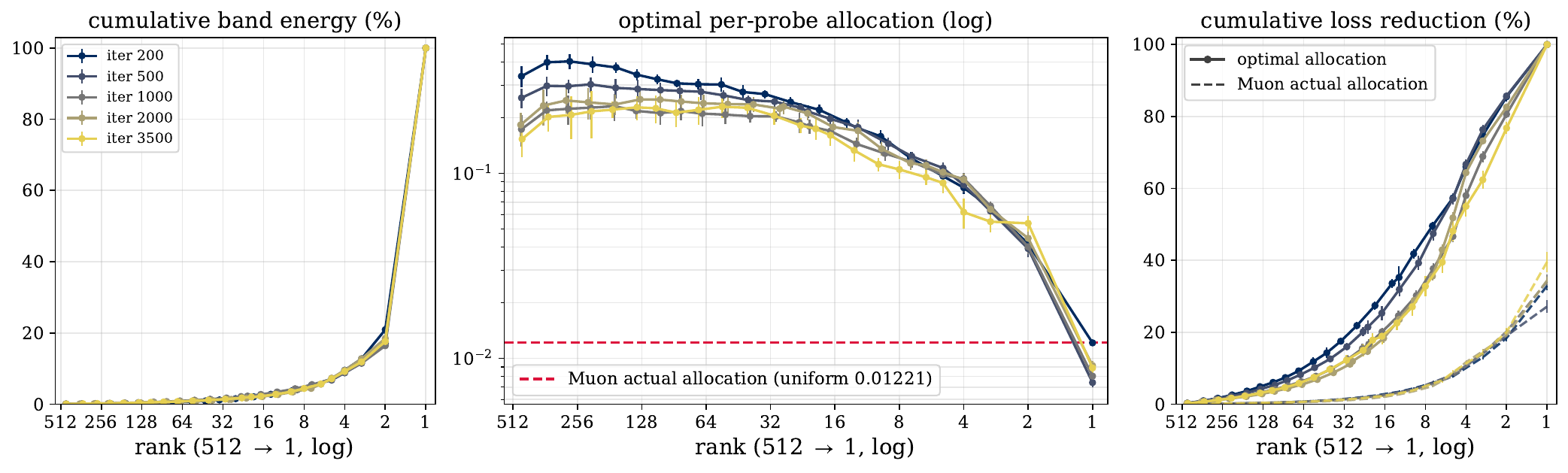}
    \caption{\textbf{Out-of-sample loss-optimal step profile across training.} Spectral probing (\cref{sec:linesearch_analysis}) of the updated \emph{momentum buffer}, along a Muon (Scion) trajectory of a 64M-parameter ``modded-nanogpt'' model (exact run in \citet{scion}), with checkpoints from iterations 200--3500 in the constant-learning-rate stage. Each point shows the mean over 20 held-out batches, and error bars indicate the standard deviation across batches. The horizontal axis runs from the deep tail (rank 512) to the head (rank 1) in log scale. \textbf{Left:} cumulative spectral energy, concentrated in the leading ranks (the head alone carries the majority). \textbf{Middle:} per-rank optimal spectral allocation $\eta^*_d$ in log scale (the LQA-estimated loss-optimal step size along each probe): this spectral profile is mostly \emph{flat across the bulk} and drops sharply near the volatile head; the head's optimal step coincides with the actual (uniform) allocation of Muon (red dashed line, $0.01221$), pinned at the Edge-of-Stability \citep{EoS}. \textbf{Right:} cumulative loss reduction (LQA estimate), accumulated for each probe from the deep tail to the head and normalised by the per-iteration optimal total, under two allocation policies: each rank at its own optimal allocation (solid, reaching $100\%$ by construction) versus every rank at the actual uniform allocation of Muon (dashed). The head-pinned allocation of Muon realises only a small fraction of this idealised per-iteration optimum.}
    \label{fig:lr_analysis}
\end{figure*}

\subsection{Observations: a stable, highly anisotropic loss-optimal step profile}

\cref{fig:lr_analysis} maps the spectral profile of the momentum buffer along the Muon (Scion) training trajectory of a 64M ``modded-nanogpt'' model \citep{moddednanogpt,scion}. Each data point is averaged over 20 held-out batches, with error bars indicating the standard deviation. We probe only 32 ranks at each checkpoint to reduce the probing cost and interpolate the resulting profile for visualisation. The foremost observation from this figure is that the measured profile as a whole is \emph{highly anisotropic yet stable}. The loss-optimal step spans nearly an order of magnitude and follows a characteristic shape that is mostly flat across the bulk and drops sharply near the head. After an initial burn-in phase of ${\sim}200$ iterations, the profile settles into this shape and persists across training, remaining nearly sample-invariant with narrow error bars across batches. This stability makes a \emph{static} structural prior viable in \cref{sec:method}, rather than an adaptively learned one. Note that this shape is not an artefact of the choice of optimiser or model scale, since measured profiles with the same characteristics also appear when probing a standard AdamW trajectory (see \cref{fig:adamw_probing} in \cref{app:adamw_probing}) and a larger model (see \cref{fig:w768_probing} in \cref{app:w768_probing}). In the remainder of this section, we refine this picture with three structural details and a further finding on its slow emergence during early training.

\textbf{The update signal concentrates in a Volatile Head operating at the Edge-of-Stability.} The leading ranks carry the dominant share of the buffer's spectral energy (the low-rank structure described in \citep{AniGrad1,SUMO}), yet these are the least stable directions in the spectrum: their loss-optimal steps are nearly an order of magnitude smaller than those of the bulk. Crucially, the actual Muon update scale coincides with this limit (red dashed line in the middle panel of \cref{fig:lr_analysis}). This provides direct empirical evidence that the head operates at the Edge-of-Stability \citep{EoS}, and that this single direction limits the maximum step size of all the spectral directions in standard Muon.

\textbf{The Tolerant Bulk is mostly flat and under-exploited.} Below the head, the loss-optimal step climbs sharply over the leading few ranks, with the largest single gap between the rank-1 and rank-2 probes, and then stays mostly \emph{flat} across the bulk at several times the scale of the head (\cref{fig:lr_analysis}), yet Muon drives all directions at the same head-pinned scale. This causes a gap across the bulk of the spectrum. By allowing the bulk directions to use their own much larger loss-optimal step sizes instead of Muon's uniform spectral allocation, a much larger per-iteration loss reduction can be achieved (the right panel of \cref{fig:lr_analysis}). Note that this cumulative loss reduction is a \emph{very} optimistic idealisation, assuming all probes to be Hessian-conjugate, so we interpret it qualitatively: the headroom is large, although its exact size is uncertain. This idealised headroom motivates bulk amplification, although cross-rank interactions determine how much can be practically recovered in a joint update.

\textbf{The near-head decline is log-rank-linear.} Replotting the same per-rank optimal allocation on a \emph{linear} scale against log rank (the left panel of \cref{fig:lgl_shape}) makes the shape of the near-head decline directly visible. It can be seen that the loss-optimal step increases approximately \emph{linearly with log rank} from the head to the bulk plateau, and stays mostly flat beyond. The measured profile can be described by two parameters, the plateau level and the log-rank slope of the decline (equivalently, its onset rank). This parametric description is stable across checkpoints and model scales in the same way as the profile itself. We emphasise that this is an empirical characterisation of the measurement, but it is directly applicable, and \cref{sec:method} applies it to the spectral allocation used by \samuon{}.

\begin{figure}[ht]
    \centering
    \includegraphics[width=0.47\textwidth]{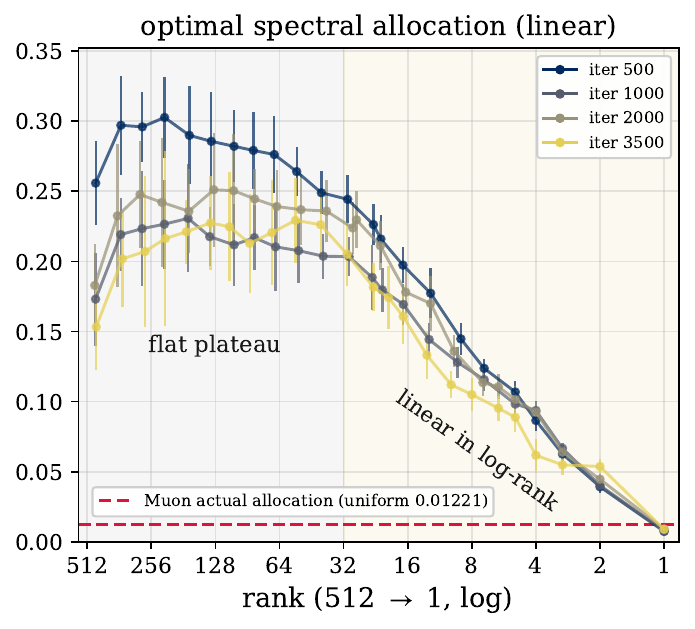}\hfill
    \includegraphics[width=0.47\textwidth]{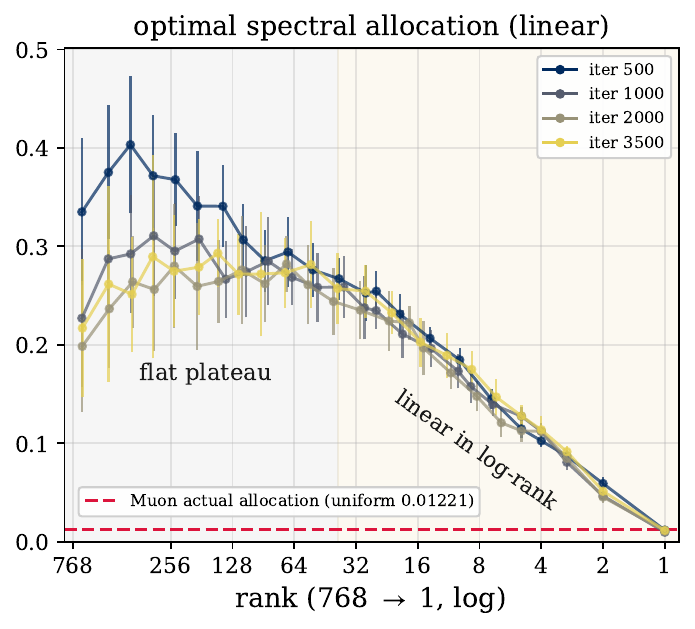}
    \caption{\textbf{The near-head decline of the loss-optimal step profile is close to linear in log rank, across model scales.} We replot the probed spectral profile with a \emph{linear} $y$-axis against log rank. \textbf{Left:} the width-512 (64M) model (the middle panel of \cref{fig:lr_analysis}); \textbf{Right:} the width-768 (124M) model (see \cref{app:w768_probing}). At both scales, the loss-optimal step decreases approximately linearly with log rank from the bulk plateau to the head. The visually estimated plateau onset shifts from rank $\sim32$ at width 512 to rank $\sim40$ at width 768, consistent with the square-root width heuristic used for \samuon{} in \cref{sec:method}.}
    \label{fig:lgl_shape}
\end{figure}

\textbf{The head gap emerges slowly over training.} The gap between the head and the rest of the spectrum is not present at initialisation. Rather, it \emph{develops} gradually over the first few hundred iterations. When we probe the same trajectory across training iterations (\cref{fig:warmup_motivation} in \cref{app:experimental_setup}), the head gap (rank-2/rank-1 optimal spectral allocation ratio, depicted in the left panel of \cref{fig:warmup_motivation}) starts from near-isotropic ($\approx 2$ at iteration 50; exact isotropy would give $1$) and slowly climbs to its final level ($\approx 5$; this measured head gap anchors the $\gamma$ search range of our experiments in \cref{sec:experiments}) within a few hundred steps, after which it remains stable. The head's spectral energy and directional curvature also concentrate over the same window at a similar pace. One consequence is that the gap that \samuon{} exploits must open before it can be exploited, which is precisely why the spectral shaping requires a \textbf{slow warmup} rather than being applied from the first step (\cref{sec:method}).

\subsection{Why Muon Outperforms Adam, and How to Improve Muon}
\label{sec:optimiser_implications}

The measured spectral profile supplies a missing convergence-side account for the success of orthogonal optimisers and quantifies the potential room for improvement.

\textbf{Spectral Re-Allocation Perspective.} Because our probing measures the loss-optimal step profile of the momentum buffer from which momentum-based optimisers generate their update steps, the comparison between optimisers reduces to a question of \emph{spectral allocation}. Momentum SGD applies the buffer directly, \textit{i.e.}\ its update $\mathbf{M}^{(t)} = \sum_d \sigma_d\, \bm u_d \bm v_d^\top$ assigns each rank a scale proportional to its singular value $\sigma_d$. Muon applies the whitened buffer $\mathrm{NS}(\mathbf{M}^{(t)}) = \sum_d \bm u_d \bm v_d^\top$, assigning a uniform scale to every singular direction. To visualise this comparison, \cref{fig:profile_cmp} overlays each optimiser's spectral allocation against the measured optimal profile. Viewing each method as a spectral allocation places SGD, Adam, and Muon on a single axis and yields a novel unified account of why Muon outperforms Adam, which in turn outperforms SGD, without relying on the separate, optimiser-specific arguments of prior work. The remainder of this section discusses the consequences.

\textbf{SGD Issues and Adam Workaround.}
The poor performance of SGD on Transformer training is well documented \citep{adamtransformer}, and our probing analysis makes its failure mode self-evident, since its signal-proportional allocation runs directly counter to the measured loss-optimal step profile. Worse still, because the Transformer gradient energy concentrates in the leading ranks \citep{AniGrad1,SUMO}, SGD allocates almost its entire update along the most unstable directions. Adam provides a workaround through coordinate-wise scaling, which implicitly dampens this allocation and explains its advantage over SGD in Transformer training. This provides a new, within-block account of the Adam advantage. While \citet{adamtransformer} attributes the Adam--SGD gap to curvature heterogeneity \emph{across} parameter blocks, our probing identifies the anisotropic spectral curvature \emph{within} parameter blocks as a complementary within-block mechanism that concurrent work by \citet{MuonCurvature} independently identifies as a driving factor. Note that the Adam workaround is spectral-agnostic, since it softens the head's weight but has no mechanism to overcome its dominance of the update spectrum. The measured allocation confirms this point: projected onto the buffer's singular basis, the true Adam update lifts the bulk's relative scale by an order of magnitude over SGD, but the head still dominates (\cref{fig:profile_cmp}).

\textbf{Advantage and Limit of Muon's Uniform Scaling.} 
Muon applies a far more direct spectral correction. By assigning a uniform scale across the spectrum, it amplifies rank $d$ by $\sigma_1/\sigma_d$ relative to the SGD spectral allocation, which greatly improves utilisation of the bulk (\cref{fig:profile_cmp}). Since the measured loss-optimal step profile is \emph{anti-correlated} with $\sigma_d$, this reallocation pushes in precisely the direction the landscape demands. The whitening greatly increases the relative update scale of the tolerant bulk, where the landscape allows far larger steps (\cref{fig:lr_analysis}). We argue that this spectral reallocation, rather than the spectral-norm stability emphasised by existing analyses, is the geometric source of the rapid convergence of Muon, directly addressing the open question raised in \cref{sec:intro}. The measured loss-optimal step is mostly flat away from the head and several times the head's Edge-of-Stability step size, whereas Muon, pinned by that single head direction, drives every rank at the head's conservative scale (\cref{fig:lr_analysis}). The allocation of Muon is therefore approximately correct in shape across the bulk but is also systematically too small in scale because of the limiting volatile head.

\textbf{Towards Structural Alignment (SAMuon).} The remaining question is how to make use of this headroom. Eigen-adaptive methods such as SOAP \citep{SOAP} can be viewed as learning this anisotropic profile through their adaptive second moments, but at a high computational and memory cost due to explicit eigendecompositions, large adaptive optimiser states, and greatly increased algorithm complexity. Our analysis shows that this expensive adaptive machinery could be avoided. The measured profile is structurally stable across batches and training stages. This means it can be distilled into a \emph{static} prior at near-zero additional cost that holds the single volatile head at the step size it already takes, and raises the remainder of the spectrum towards the tolerance the probe measures, either uniformly or following the measured log-rank-linear shape (the \samuonlite{} and \samuon{} curves in \cref{fig:profile_cmp}). \cref{sec:method} implements this idea in SAMuon and SAMuon-lite, using low-cost, transient low-rank estimates of the head region.

\begin{figure*}[t]
    \centering
    \includegraphics[width=\textwidth]{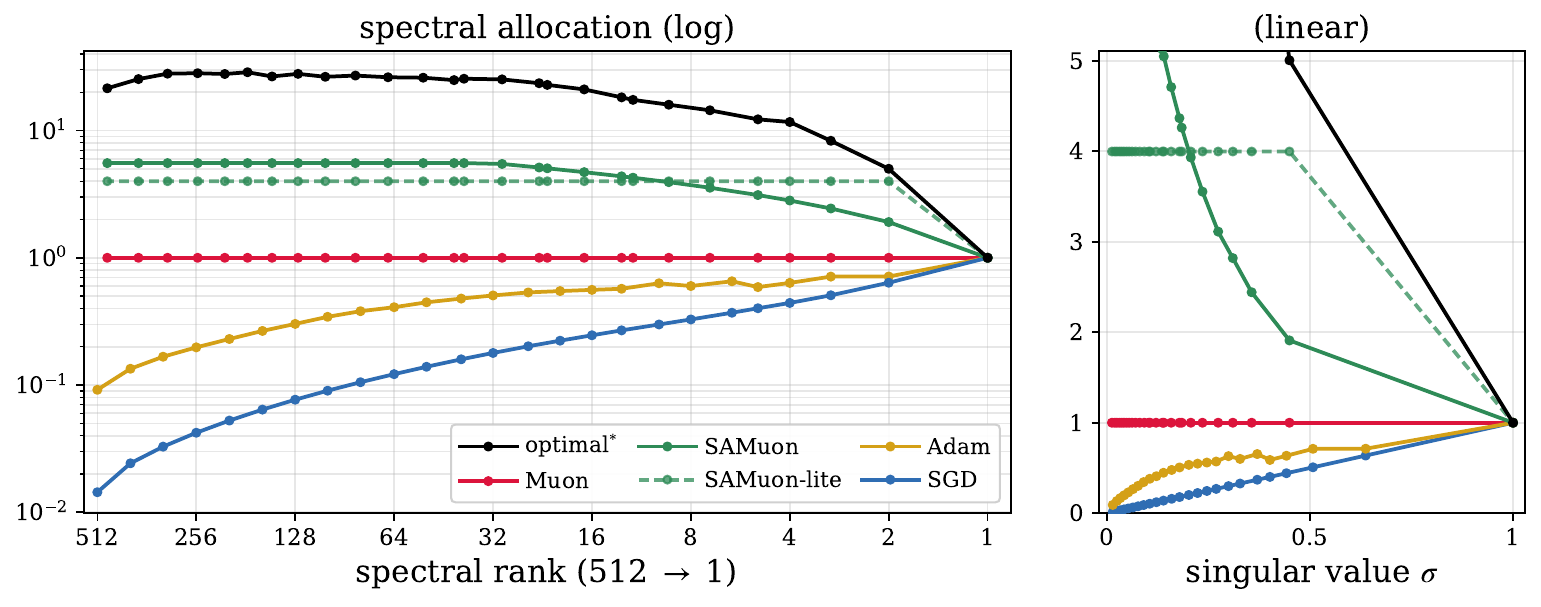}
    \caption{\textbf{Optimisers differ in how they scale the momentum-buffer spectrum, and \samuon{} / \samuonlite{} match the measured optimum most closely.} The updates generated by different optimisers implicitly assign a per-rank scaling to the singular components of the \emph{momentum buffer}, and we overlay these allocations against the \emph{optimal$^{*}$} profile (the iteration-1000 out-of-sample loss-optimal step size profile of Muon from \cref{fig:lr_analysis}). All profiles are head-normalised to $1$ and evaluated at the iteration-1000 checkpoint, with the Adam allocation drawn from the AdamW trajectory (see \cref{app:profile_cmp_details}). \textbf{Left:} scaling per spectral rank (log axes). \textbf{Right:} the same allocations as a function of normalised singular value $\sigma$ (linear axes). \samuonlite{} ($\gamma = 5$) raises the bulk to a single uniform level well above the Muon scale while holding the head at the same scale as Muon, and \samuon{} ($\gamma = 7.07$) additionally follows the measured log-rank-linear decline over the leading ranks, the closest shape match to the measured optimum profile. From SGD $\to$ Adam $\to$ Muon $\to$ \samuonlite{} $\to$ \samuon{}, each successive optimiser aligns its allocation more closely with what the spectral profile suggests.}
    \label{fig:profile_cmp}
\end{figure*}

\section{Spectral-Aware Muon (SAMuon)}
\label{sec:method}

Our spectral probing shows that the out-of-sample optimal step is mostly \emph{flat across the bulk} of the spectrum, far above the much smaller step the volatile head tolerates (\cref{fig:lr_analysis}). The head sits at the Edge-of-Stability and limits the global step size: the mostly flat bulk could absorb a uniformly larger update, but Muon, by whitening every direction to the same scale, can only use the head's conservative step for the whole spectrum. To address this issue, we propose \textbf{Spectral-Aware Muon (SAMuon)}. As detailed in \cref{alg:samuon}, \samuon{} retains the momentum and whitening mechanisms of Muon but applies a \emph{head-anchored} spectral allocation to the whitened update: it \emph{holds the head at the same scale as Muon} while increasing the rest of the spectrum towards its measured tolerance, realised through a low-cost, transient low-rank estimate of the head region that adds no persistent optimiser state. We instantiate this allocation into two variants: \samuon{}, which follows the measured log-rank-linear shape of \cref{fig:lgl_shape}; and \samuonlite{}, which is a two-level simplification that runs at nearly the same wall-clock speed as Muon.

\subsection{Head-anchored spectral allocation: from two levels to the measured profile}
\label{subsec:spectral_scaling}

\textbf{Momentum Integration.}
Both instantiations maintain a single momentum buffer per 2D weight matrix, updated exactly as in the Muon baseline (\cref{eqn:muon_momentum}). The spectral modification below changes only how the buffer is orthogonalised, leaving its temporal dynamics untouched.

\textbf{Head-anchored scaling.}
\samuon{} and \samuonlite{} distil the measured profile of \cref{sec:linesearch_analysis} (mostly flat across the bulk, log-rank-linear near the head, see \cref{fig:lgl_shape}) into a static allocation that anchors the head and rescales the rest of the spectrum. Let $\mathrm{NS}(\mathbf{M}^{(t)}) \approx \sum_i \bm u_i \bm v_i^\top$ be the whitened update used in Muon; under exact whitening, it drives every singular direction at unit scale. Let $(\bm u_i, \bm v_i)_{i \le k}$ be the leading $k$ singular pairs of the buffer $\mathbf{M}^{(t)}$. The SAMuon update to each weight matrix $\mathbf{O}_\text{SA}^{(t)}$ raises the whole whitened momentum buffer to scale $\gamma_t$ and then restores the leading $k$ directions to scales $s_i(t)$:
\begin{equation}
\label{eqn:samuon_update}
    \mathbf{O}_\text{SA}^{(t)} \;=\; \gamma_t\, \mathrm{NS}\!\big(\mathbf{M}^{(t)}\big) \;-\; \sum_{i=1}^{k} \big(\gamma_t - s_i(t)\big)\, \bm u_i \bm v_i^\top,
\end{equation}
Under exact whitening, this allocates rank $i \le k$ the coefficient $s_i(t)$ and every other direction the coefficient $\gamma_t$. The target profile anchors the head at the same unit scale as Muon, $s_1(t) = 1$, so it remains at the Edge-of-Stability step size that whitening already respects. With the finite $4 \times 3$ Newton--Schulz implementation, these target allocations are approximate; its response is quantified in \cref{app:ns_parameters}. The scales rise monotonically to the level of the tail, $\gamma_t$, with $1 = s_1(t) \le s_i(t) \le \gamma_t$. Both variants start from the uniform profile of Muon and warm up gradually to the target profile. A warmup weight $w_t \in [0,1]$, rising from $0$ to $1$ over the warmup horizon (cosine schedule, see \cref{app:experimental_setup}), sets $\gamma_t = 1 + (\gamma - 1)\,w_t$ and $s_i(t) = 1 + (s_i - 1)\,w_t$ for target scales $s_i$. Training begins at exactly Muon ($w_t = 0$ collapses \cref{eqn:samuon_update} to $\mathrm{NS}(\mathbf{M}^{(t)})$), and the target anisotropic profile takes effect only after the head-bulk gap in the spectral profile has emerged (\cref{app:experimental_setup}). Setting $\gamma = 1$ recovers the standard Muon update exactly for SAMuon / SAMuon-lite, making the Muon baseline a strict ablation.

\textbf{\samuon{}: the measured log-rank-linear profile.}
SAMuon follows the empirical shape of \cref{fig:lgl_shape} directly so that the target scales increase linearly in log rank from the head to the level of the tail,
\begin{equation}
\label{eqn:lgl_profile}
    s_i \;=\; 1 + (\gamma - 1)\,\frac{\log i}{\log k}, \quad 1 \le i \le k \;\; (k \ge 2),
\end{equation}
and all ranks beyond $k$ receive the fixed tail scale $\gamma$. The value of $k$ is set by the measured onset of the plateau, $k = 32$ on the width-512 probing model, and transfers across widths by a square-root rule inspired by \cref{fig:lgl_shape}, \textit{i.e.}\ $k(d_\text{model}) = \left\lfloor 32\sqrt{d_\text{model}/512}\,\right\rfloor$. The leading $k$ singular pairs are obtained from a randomised low-rank SVD of the buffer at each update step (\texttt{torch.svd\_lowrank}; \citealp{lowrankSVD}). We emphasise that this shape and scaling rule are an \emph{empirical} distillation of the measurement and not a theoretically derived optimum. We defer the theoretical investigation of this measurement to future work. Under this rule, \samuon{} attains the best or joint-best final loss at every batch size and model scale (\cref{sec:main_results}).

\textbf{\samuonlite{}: a two-level simplification at Muon speed.}
At $k = 1$, \samuonlite{} does not need a ramp (\cref{eqn:lgl_profile} is not invoked), and only the anchor $s_1 = 1$ remains with a uniformly boosted remainder. This deliberately coarse two-level allocation captures the profile's dominant feature of a single volatile head with a far more tolerant bulk. The single leading pair $(\bm u_1, \bm v_1)$ is obtained from a few steps of power iteration on $\mathbf{M}^{(t)}$ (reusing the buffer, with no explicit SVD) at $O(mn)$ cost, a handful of matrix-vector products, so \samuonlite{} runs at nearly the same wall-clock speed as Muon. For both SAMuon variants, the weights are then updated as $\mathbf{W}^{(t+1)} = \mathbf{W}^{(t)} - \eta\,\kappa\,\mathbf{O}_\text{SA}^{(t)}$.

\begin{algorithm}[ht]
   \caption{SAMuon / SAMuon-lite, shown for a single 2D weight matrix}
   \label{alg:samuon}
\begin{algorithmic}[1]
   \small
   \State {\bfseries Input:} initial weights $\mathbf{W}^{(1)}$, momentum buffer $\mathbf{M}^{(0)} = \mathbf{0}$, Scion matrix scaling $\kappa=\sqrt{d_{\mathrm{out}}/d_{\mathrm{in}}}$
   \State {\bfseries Hyperparameters:} $\eta$ (effective LR), $\mu$ (momentum), $\gamma$ (tail boost), $T$ (total steps); profile rank $k$ (width rule, \cref{app:experimental_setup}) and target scales $\{s_i\}_{i=1}^{k}$ (\cref{eqn:lgl_profile}) for \samuon{} only
   \For{$t = 1, \dots, T$}
      \State Compute gradient $\mathbf{G}^{(t)}$ on batch $\mathcal{B}_t$
      \State $\mathbf{M}^{(t)} \leftarrow \mu\, \mathbf{M}^{(t-1)} + (1-\mu)\, \mathbf{G}^{(t)}$
      \State $\gamma_t \leftarrow 1 + (\gamma - 1)\, w_t$ \Comment{cosine-scheduled spectral warmup}
      \State \textbf{Variant \#1: \samuon{}}  (log-rank-linear profile):
      \State \hspace{1em} $(\bm{u}_i, \bm{v}_i)_{i \le k} \leftarrow \mathrm{LowRankSVD}(\mathbf{M}^{(t)},\, k)$ \Comment{\texttt{torch.svd\_lowrank()}}
      \State \hspace{1em} $s_i(t) \leftarrow 1 + (s_i - 1)\, w_t$, for $i = 1, \dots, k$ \Comment{head region spectral warmup}
      \State \hspace{1em} $\mathbf{O}_\text{SA}^{(t)} \leftarrow \gamma_t\, \mathrm{NS}(\mathbf{M}^{(t)}) - \sum_{i=1}^{k} (\gamma_t - s_i(t))\, \bm{u}_i \bm{v}_i^\top$ \Comment{\cref{eqn:samuon_update}}
      \State \textbf{Variant \#2: \samuonlite{}} ($k = 1$, simple two-level profile):
      \State \hspace{1em} $(\bm{u}_1, \bm{v}_1) \leftarrow \mathrm{PowerIter}(\mathbf{M}^{(t)})$ 
      \State \hspace{1em} $\mathbf{O}_\text{SA}^{(t)} \leftarrow \gamma_t\, \mathrm{NS}(\mathbf{M}^{(t)}) - (\gamma_t - 1)\, \bm{u}_1 \bm{v}_1^\top$ \Comment{\cref{eqn:samuon_update} with $k=1$}
      \State $\mathbf{W}^{(t+1)} \leftarrow \mathbf{W}^{(t)} - \eta\,\kappa\, \mathbf{O}_\text{SA}^{(t)}$ 
   \EndFor
\end{algorithmic}
\end{algorithm}

\subsection{Algorithm analysis: complexity and convergence guarantee}
\label{subsec:complexity}

\samuon{} preserves the core efficiency of standard orthogonal optimisers. Memory-wise, both instantiations maintain exactly one dense momentum buffer per weight matrix, identical to Muon ($mn$) and strictly smaller than Adam ($2mn$) and SOAP ($6mn$). The head estimate adds only a \emph{transient} rank-$k$ factor pair ($k(m+n) \ll mn$), not persistent state, so SAMuon remains the most memory-frugal orthogonal optimiser of the recent Muon variants. In FLOPs, the per-step overhead over Muon is negligible for both: $O(mn)$ for the power iteration of \samuonlite{} and $O(mnk)$ for the rank-$k$ randomised SVD of \samuon{}, with $k$ on the order of $\sqrt{d_\text{model}} \ll r$, both small beside the $O(mnr)$ Newton-Schulz whitening. Wall-clock behaviour, however, currently separates the two. Measured per training iteration on the 1B model (\cref{app:wallclock}), the power iteration of \samuonlite{} adds $0.5\%$ to the Muon iteration time ($5.3\%$ of the NS whitening), so it runs at nearly the same speed as Muon, whereas the \texttt{torch.svd\_lowrank} call of \samuon{} adds $7.4\%$, nearly as much as the Newton-Schulz whitening itself. It is left for future work to improve the current \texttt{svd\_lowrank} implementation, which could greatly reduce its wall-clock cost given the low theoretical FLOP count.

\textbf{Theoretical Convergence.} We provide a convergence guarantee for the idealised exact-whitening versions of both variants (assuming perfect Newton-Schulz whitening) in \cref{app:convergence_samuon} (\cref{thm:samuon_convergence}). For any head-anchored allocation with scales in $[1, \gamma]$, the guarantee preserves the convergence orders of Muon up to $\gamma$-dependent constants: $\mathcal{O}(T^{-1/4})$ in the stochastic setting and $\mathcal{O}(T^{-1/2})$ in the deterministic setting \citep{MuonConvergence}, covering both proposed algorithms.

\section{Experiments}
\label{sec:experiments}

We evaluate the SAMuon and SAMuon-lite instantiations on Large Language Model (LLM) pretraining, following the rigorous experimental protocol established in \citet{scion} for training models with norm-constrained updates. Both instantiations are implemented directly on top of the codebase for Scion \citep{scion}.

\subsection{Experimental setup}
\label{sec:experimental_setup}

\textbf{Architecture and Training Protocol.} 
Our experimental setup follows the exact configuration in \citet{scion} to ensure a rigorous and fair comparison with the baselines. We train ``modded-nanogpt'' models \citep{moddednanogpt,scion} at three scales: 124M (width 768), 300M (width 1280), and 1B (width 2560) parameters. The models deviate from the standard GPT-2 by adopting modern architectural improvements: replacing LayerNorm with RMSNorm; utilising Rotary Positional Embeddings (RoPE) \citep{RoPE} in place of absolute embeddings; and removing bias terms to maximise arithmetic intensity.
All runs train on the FineWeb dataset \citep{fineweb} (100B-token sample). The 124M and 300M models use a 10B-token budget and the 1B model a 20B-token budget, with the iteration count matched to the budget for each batch size. The batch size is swept over 1024, 2048, and 4096 sequences. For the 124M model the budget is roughly $4\times$ the Chinchilla-optimal allocation \citep{chinchilla} (${\sim}1.7\times$ at 300M and ${\sim}1\times$ at 1B), so the 124M runs train well into the over-trained regime. We employ a \textit{trapezoidal} learning rate schedule consisting of a constant phase followed by a linear decay over the final 28.5\% of training tokens (tuned schedule for Muon by \citet{scion}). A 5\% linear warmup phase is included for the AdamW baseline, and a 30\% spectral warmup phase (for $w_t$ in \cref{alg:samuon}) is applied to the SAMuon variants at the beginning of training. Each optimiser configuration is evaluated with a single seed. Within each model-scale and batch-size setting, all optimiser runs use the same model initialisation and identical data ordering. Finally, all runs execute on NVIDIA A100 (80GB) GPUs in bfloat16 mixed precision.

\textbf{Baselines.} 
We compare against two primary optimisers: \textbf{AdamW} \citep{AdamW}, the standard adaptive baseline ($\beta_1=0.9, \beta_2=0.95$) used in most LLM training setups; and \textbf{Muon} (the Unconstrained \textbf{Scion} implementation by \citet{scion}). We chose Scion over the original Muon implementation due to its simpler hyperparameter tuning and better performance. A detailed discussion is provided in \cref{app:scion_over_muon}. Since \samuon{} is implemented directly on the Scion codebase, Scion serves as the strict ablation baseline (equivalent to either instantiation with $\gamma=1$).

\textbf{Hyperparameter Tuning Protocol.}
All hyperparameters are tuned per batch size on the 124M model, and transferred across model scales following the convention of \citet{scion}. For AdamW we sweep the learning rate on a half-power-of-two logarithmic grid. For Muon we sweep the \emph{spectral radius} (\texttt{radius\_2d}), the learning-rate factor Scion applies to the 2D, Muon-effective modules, which plays for orthogonal updates the role that the learning rate plays for AdamW, over a $\sqrt{2}$-spaced grid $\{35.36, 50, 70.71, 100, 141.42\}$. The tuned Muon optimum is $70.71/50/70.71$ at batch sizes $1024/2048/4096$. Both \samuon{} variants run at the fixed spectral radius $50$, the Scion default, and they sweep the tail boost $\gamma$ over $\{3.54, 5, 7.07, 10, 14.14\}$. The value of $k$ for \samuon{} follows the fixed width rule ($k = 39/50/71$ at widths $768/1280/2560$; \cref{app:experimental_setup}) and is not tuned. Finally, the hyperparameters are transferred directly across larger model scales, following the transfer rule found in \citet{scion} which uses a constant optimal radius for Scion and a halved learning rate per model-scale step for AdamW. $\gamma$ is assumed to be constant across model scales and is transferred without retuning, which makes the \samuon{} variants relatively under-tuned compared to the baselines (see discussion in \cref{sec:discussion}). Refer to \cref{app:experimental_setup} for the full hyperparameter sweeps and other experimental details.

\subsection{Main results}
\label{sec:main_results}
\cref{tab:full-scaling-grid} reports the final validation loss and the token-efficiency improvement over the model-scale (124M/300M/1B) $\times$ batch-size (1024/2048/4096) grid. Analysis of the results is provided below.

\textbf{\samuon{} variants outperform the Muon (Scion) baseline by a large margin across model scales and batch sizes.} Both variants attain a lower final validation loss than the tuned Muon baseline in all trained cells. \samuon{} improves on that baseline by $0.0187$--$0.0295$ at 124M, $0.0196$--$0.0228$ at 300M, and $0.0137$--$0.0191$ at 1B, which translates to a token-efficiency improvement of $13.3\%$--$24.0\%$ across the grid. For the simpler \samuonlite{} variant, the corresponding improvements are $0.0184$--$0.0275$, $0.0173$--$0.0203$, and $0.0134$--$0.0173$, which translate to a token-efficiency improvement of $13.3\%$--$22.1\%$. This demonstrates that both \samuon{} variants make effective use of the headroom identified in the spectral profile and achieve better convergence across models trained at the Chinchilla-optimal token budget or beyond. 

\textbf{\samuon{} variants' improvement over Muon grows with batch size.}
A thorough batch-size sweep is provided for the 124M model in the first three rows of \cref{tab:full-scaling-grid}, and a partial batch-size sweep is provided for larger models in the rest of the same table. A consistent trend can be observed across model scales: the token-efficiency improvements of \samuon{} and \samuonlite{} over Muon generally grow with batch size. They increase from $20.3\%$ and $20.3\%$ (1024) to $24.0\%$ and $22.1\%$ (4096) on 124M, from $17.7\%$ and $15.9\%$ (1024) to $18.8\%$ and $17.1\%$ (2048) on 300M, and from $13.3\%$ and $13.3\%$ (1024) to $15.6\%$ and $14.4\%$ (4096) on 1B. One possible explanation is that larger batches provide a less noisy estimate of the spectral structure in the momentum buffer.

\textbf{\samuon{} matches or outperforms \samuonlite{}, and the gap generally widens with batch size.} Across the model-scale and batch-size grid, \samuon{} attains a final validation loss at least as low as \samuonlite{} in all cells in \cref{tab:full-scaling-grid}, with a difference of $0.0003$--$0.0027$ at 124M, $0.0023$--$0.0025$ at 300M, and $0.0003$--$0.0018$ at 1B, which translates to a small token-efficiency advantage of $0.0$--$2.3$ (percentage) points. At batch size 1024 on the 124M and 1B models the two variants are effectively tied, with a loss difference of $0.0003$ and an equal token-efficiency improvement. Elsewhere the margin favours \samuon{}, which suggests that the more precise log-rank-linear allocation provides an edge in optimisation by following a profile closer to the observation in our spectral probing analysis. Additionally, it can be observed that the gap between \samuon{} and \samuonlite{} generally widens with batch size, which grows from $0.0$ (1024) to $1.9$ (percentage) points (4096) on 124M, from $0.0$ (1024) to $1.2$ points (4096) on 1B, and holds level on 300M at $1.8$ (1024) and $1.7$ points (2048). One possible explanation is that larger batches provide a less noisy estimate of the spectral structure in the momentum buffer, allowing \samuon{} to benefit more from its finer-grained allocation. Note that the simpler two-level allocation of \samuonlite{} remains a cost-effective alternative that recovers most of \samuon{}'s token-efficiency improvement over Muon, while offering a simpler implementation and lower per-iteration wall-clock cost.

\begin{table}[t]
\caption{\textbf{Final validation loss and token-efficiency improvement across the model-scale $\times$ batch-size grid.} The hyperparameters for each optimiser are tuned for every batch size setup on the 124M model. The larger model scales (300M and 1B) inherit the tuned hyperparameters from the corresponding batch size of the 124M model, following the scaling convention of \citet{scion} (see \cref{sec:experimental_setup}). The final loss is reported at the end of the training budget (10B tokens for 124M and 300M, 20B tokens for 1B). The token-efficiency improvement (improv.) is the token budget saving $1 - T^{\star}/N$ of a \samuon{} variant over the Muon baseline, where $N$ is the schedule length of the tuned Muon baseline and $T^{\star}$ the shortest schedule with which the variant still reaches the same final loss (\cref{app:speedup}). The best (unrounded) loss and token-efficiency improvement in each row are highlighted in bold.}

\vspace{4pt}
\centering
\setlength{\tabcolsep}{5pt}
\setlength{\aboverulesep}{0.35ex}
\setlength{\belowrulesep}{0.5ex}
\begin{tabular}{llc rr rr rr}
\toprule
\multirow{2}{*}{Model} & \multirow{2}{*}{Batch} & \multirow{2}{*}{Tokens} & \multirow{2}{*}{AdamW} & \multirow{2}{*}{Muon} & \multicolumn{2}{c}{\samuon{}} & \multicolumn{2}{c}{\samuonlite{}} \\[-0.15ex]
\cmidrule(lr){6-7}\cmidrule(lr){8-9}
 & & & & & \multicolumn{1}{c}{loss} & \multicolumn{1}{c}{improv.} & \multicolumn{1}{c}{loss} & \multicolumn{1}{c}{improv.} \\
\midrule
124M & 1024 & 10B & 3.2105 & 3.1622 & \textbf{3.1435} & \textbf{20.3\%} & 3.1438 & 20.3\% \\
 & 2048 & 10B & 3.2410 & 3.1726 & \textbf{3.1484} & \textbf{23.8\%} & 3.1511 & 21.5\% \\
 & 4096 & 10B & 3.3050 & 3.1953 & \textbf{3.1658} & \textbf{24.0\%} & 3.1678 & 22.1\% \\
\midrule
300M & 1024 & 10B & 3.0297 & 2.9836 & \textbf{2.9640} & \textbf{17.7\%} & 2.9663 & 15.9\% \\
 & 2048 & 10B & 3.0658 & 2.9941 & \textbf{2.9713} & \textbf{18.8\%} & 2.9738 & 17.1\% \\
\midrule
1B & 1024 & 20B & 2.7599 & 2.7251 & \textbf{2.7114} & \textbf{13.3\%} & 2.7117 & 13.3\% \\
 & 4096 & 20B & 2.8067 & 2.7413 & \textbf{2.7222} & \textbf{15.6\%} & 2.7240 & 14.4\% \\
\bottomrule
\end{tabular}

\label{tab:full-scaling-grid}
\end{table}

\section{Discussion}
\label{sec:discussion}
\textbf{The offline probing analysis for new optimiser design.} The proposed out-of-sample spectral probing based on saved model checkpoints turns optimiser design into a problem that can be efficiently measured. The SAMuon variants show that a \emph{static}, head-anchored prior derived from the measurement can provide large improvements on Muon at near-zero additional FLOPs. We believe this shows a promising direction for future work on optimiser design, where researchers can probe the loss landscape with saved checkpoints efficiently. This process can generate rich design signal (\textit{e.g.}\ to discover the under-exploited spectral capacity in Muon in this paper) and save the cost of thorough full-training sweeps in early design stages.

\textbf{SAMuon variants are currently under-tuned relative to the Muon (Scion) baseline.} 
The SAMuon variants reported in \cref{sec:main_results} are under-tuned relative to their Muon baselines in two ways, which might understate their performance. Firstly, the SAMuon variants follow the same schedule as the Muon baseline in \citet{scion}, which is optimised for Muon. Our preliminary investigation has shown that the SAMuon variants can benefit from a much longer linear decay phase. Properly tuning the learning rate schedule for the SAMuon variants specifically would likely improve their performance. Secondly, the hyperparameters for the SAMuon variants are transferred directly from the tuned hyperparameters for the 124M model to larger model-scale setups without re-tuning, while the hyperparameters for Muon and AdamW are transferred across model scales by following the scaling rules tuned and discovered in \citet{scion}. This difference makes the hyperparameters for the SAMuon variants relatively under-tuned for larger model-scale setups. We leave a more thorough hyperparameter/schedule tuning for the SAMuon variants to future work.

\textbf{Is there a better spectral profile?} \samuonlite{} adopts the simplest two-level profile, which achieves a $13.3\%$--$22.1\%$ token-efficiency improvement in training. Meanwhile, the richer log-rank-linear profile adopted in \samuon{} provides a small yet consistent improvement on top of \samuonlite{} of up to $2.3$ (percentage) points across all training setups. This is evidence that following the measured optimal spectral profile is beneficial, though the margin is small relative to the gap over Muon. The natural next step is to ask whether there is a better spectral profile than the measured one. First and foremost, the current loss-optimal step profile does not consider cross-probe interactions. A more refined scaling profile/analysis that takes cross-probe curvature interactions into account could yield a much better spectral profile. Another possible direction is to investigate whether the profile should adapt to its drift over the training trajectory, beyond the current early warmup. A better-scheduled or dynamic allocation mechanism that addresses this long-term drift could be beneficial. Finally, the current profile only considers intra-block spectral allocation, where the spectral profile is averaged over all probed layers/modules in the model. A preliminary investigation of per-module gradient alignment in \cref{app:sign_alignment} shows clear inter-block differences. Thus, a more refined profile that assigns block-dependent allocations could be beneficial. We leave these questions for future work. 

\section{Conclusion}
\label{sec:conclusion}
In this work, we contributed an \emph{out-of-sample} spectral analysis of the Transformer pretraining loss landscape in order to understand why whitening the momentum buffer accelerates convergence so effectively. Probing the spectrum of the momentum buffer along real training trajectories, we mapped the optimal per-rank step size that the loss landscape tolerates along each singular direction and found it to have a stable, characteristic shape. This shape is mostly flat across the bulk of the spectrum, declining sharply near a single volatile head which sits at the Edge-of-Stability and allows a far smaller step size than the rest, with the decline tracing a close-to-linear ramp in log rank over the leading ranks. Viewed as a spectral allocation, this measured profile gives a unified account of why Muon outperforms Adam, which in turn outperforms SGD. It shows the limit of uniform whitening, which is constrained by the volatile head and under-exploits the more tolerant bulk of the spectrum.

Guided by this measured profile, we proposed Spectral-Aware Muon (\samuon{}), which employs \emph{static}, head-anchored priors on the whitened update that hold the volatile head at the same scale as Muon and increase the remainder of the spectrum towards its measured tolerance. \samuon{} implements the empirically measured log-rank-linear shape using a rank-$k$ randomised SVD of the momentum buffer, while the simplified \samuonlite{} keeps a coarse two-level allocation with power iteration and runs at nearly the same wall-clock speed as Muon. Neither adds persistent optimiser state over Muon, and a single convergence guarantee covers the idealised exact-whitening versions of both variants. Under the shared tuning-and-transfer protocol, with token budgets from approximately one to four times Chinchilla-optimal, \samuon{} reaches the same final loss as the tuned Muon (Scion) baseline with a $13.3\%$--$24.0\%$ token-efficiency improvement in all trained model and batch size combinations. \samuonlite{} matches or nearly matches that improvement, and the improvement generally grows with batch size.
These gains are a validation of the analysis. Aligning an optimiser's update with the measured loss-optimal step profile translates directly into better convergence. Furthermore, \samuon{} matches or improves on the two-level \samuonlite{} in every setting, with the clearest margins at larger batch sizes. This suggests that the complete \emph{spectral profile}, not merely the head/bulk split, carries a meaningful signal. We hope the profile and the probing methodology behind it serve as a reusable starting point for spectral-allocation optimiser design.

\bibliography{references}
\bibliographystyle{plainnat}

\newpage
\appendix
% cleveref: tag appendix sectioning so \cref says ``Appendix'' (applies only to labels defined below, so main-text \cref{sec:...} stays ``Section'')
\crefalias{section}{appendix}
\crefalias{subsection}{appendix}

\section{Spectral Probing: Assumptions and Validation}
\label{app:spectral_probing_assumption}
This section states the assumptions and retention conditions required to interpret the local-quadratic estimate in \cref{subsec:probing_method}, reports the filtering of probes that fail these conditions, and validates the local-quadratic approximation against an exact line-search reference.

\begin{assumption}[Local Quadraticity]
\label{assum:local_quadratic}
Along each probe direction $\{\bm\theta_t + \eta\,\Delta\bm\theta^{(d)} : \eta \ge 0\}$, the batch loss is well-approximated by its second-order Taylor expansion (\cref{eqn:lqa}) over the step range containing the loss minimum.
\end{assumption}

\begin{assumption}[Held-Out Descent Direction]
\label{assum:heldout_descent}
For every retained rank-$d$ probe, the directional derivative of the held-out loss is negative, $\bm g^\top \Delta\bm\theta^{(d)} < 0$, where $\bm g$ is the gradient of the held-out batch loss at $\bm\theta_t$. Thus, a positive step along the probe direction initially decreases the held-out loss.
\end{assumption}

\begin{assumption}[Positive Directional Curvature]
\label{assum:positive_curv}
The probed curvature $\Delta\bm\theta^{(d)\top}\mathbf{H}\,\Delta\bm\theta^{(d)}$ is positive. Together with \cref{assum:heldout_descent}, this ensures that the quadratic in \cref{eqn:lqa} has a positive interior minimum. All probes retained in our analysis satisfy this condition.
\end{assumption}

\textbf{Filtering ill-behaved probes.}
The LQA yields a positive interior optimum under \cref{assum:heldout_descent,assum:positive_curv}. We therefore discard any probe that fails either check before aggregating or plotting the results. Across the 64M and 124M out-of-sample spectral probing analyses, 31 of the $6{,}400$ sampled rank probes ($0.48\%$) fail the descent-direction check, while two additional rank probes ($0.03\%$) have negative curvature, giving 33 filtered rank probes ($0.52\%$) in total. The two negative-curvature cases occur at the deepest ranks late in training and are numerically negligible, with curvature magnitudes no greater than $4.0 \times 10^{-7}$; the non-descent cases are likewise concentrated late in training. Thus, every probe included in the analysis satisfies both conditions, while the filtering affects only a tiny fraction of the measurements.

\textbf{Validation against an exact reference: golden-section line search.}
For every LQA probe, the finite-difference curvature estimate of \cref{eqn:fd_curv} uses the fixed displacement $\varepsilon=0.3$. To further test \cref{assum:local_quadratic} directly, we measure the exact optimal step of a probe by a one-dimensional line search along the probe direction, free from any quadratic assumptions. An exponential coarse search first brackets the loss minimum, and then a golden-section refinement in $\log_{10}$-space locates it to a target precision. The search operates purely on forward evaluations of the batch loss but requires tens of evaluations per probe, compared with the single additional forward evaluation required by the LQA estimate of \cref{eqn:fd_curv}. 
We report the loss-optimal step in update-norm units as $s_d^* = \|\eta_d^*\Delta\bm\theta^{(d)}\|_2 = \eta_d^*\|\Delta\bm\theta^{(d)}\|_2$, where the Scion matrix scales $\kappa_j$ are already included in $\Delta\bm\theta^{(d)}$. \cref{fig:lqa_validation} compares $s_d^*$ estimated by the LQA with the line-search reference over 24 checkpoint--batch analyses spanning training stages and both Muon (Scion) and AdamW trajectories (773 probes in total). The two agree closely across more than four orders of magnitude of update norm: the median relative deviation remains near $5\%$ throughout the operationally relevant range ($s_d^* \lesssim 10$), which covers the volatile head and the bulk of the spectrum. Deviations grow only for the flattest tail directions ($s_d^* \gtrsim 10^{2}$), where the loss basin is extremely shallow, so the precise location of the minimum is both harder for the reference search to pin down and less consequential for the resulting profile. This close agreement confirms that the pretraining loss is locally near-quadratic along the probed directions (across training stages and update orientations), justifying the single-forward-pass LQA estimate used in our spectral analysis.

\begin{figure}[ht]
    \centering
    \includegraphics[width=\textwidth]{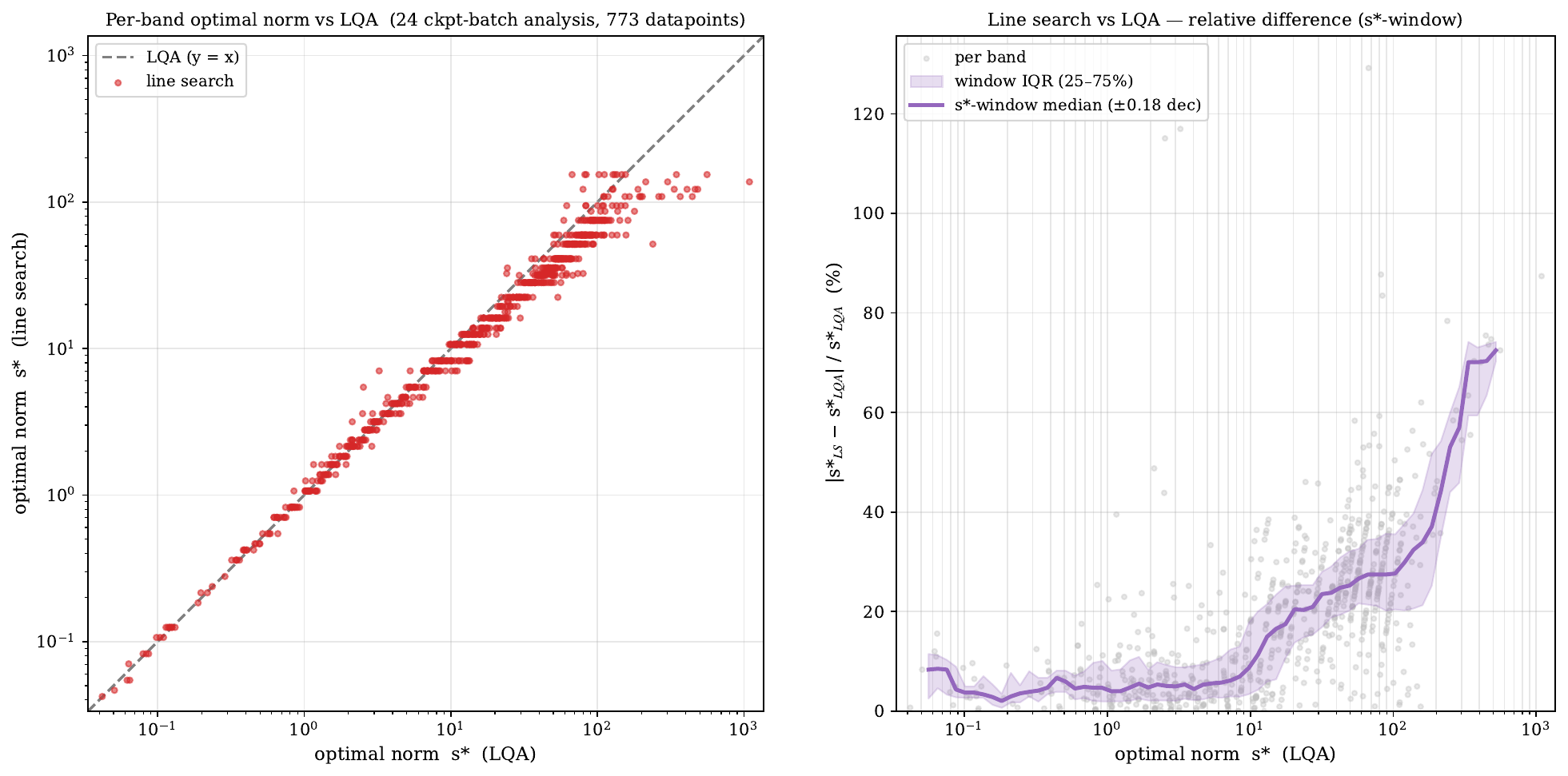}
    \caption{\textbf{Validation of the local-quadratic estimate against the exact line search.} \textbf{Left:} model-wide loss-optimal update norm $s_d^*$, estimated by the LQA (x-axis) and by the golden-section line-search reference (y-axis), for the 773 probe directions sampled from 24 checkpoint--batch analyses spanning all training stages and both Muon and AdamW trajectories; the dashed line marks perfect agreement ($y=x$). \textbf{Right:} relative difference between the two as a function of the LQA estimate, with the windowed median and interquartile range overlaid. The median deviation stays near $5\%$ over the operationally relevant range, rising only for extremely flat directions where the loss minimum is shallow.}
    \label{fig:lqa_validation}
\end{figure}

\section{Spectral Probing of an AdamW Trajectory}
\label{app:adamw_probing}

The spectral profile in \cref{fig:lr_analysis} is measured on checkpoints produced by a Muon run, which raises a natural question: is the measured structure a property of the Transformer loss landscape, or an artefact of the particular trajectory that the update rule of Muon carves through it? To answer this, we repeat the complete out-of-sample probing analysis of \cref{sec:linesearch_analysis} on a standard AdamW trajectory of the same 64M-parameter model under the identical protocol. The gradient of a gradient batch $\mathcal{B}_{\mathrm{g}}$ is integrated into the first-moment (momentum) buffer stored at each AdamW checkpoint, and the same Scion-scaled per-rank probes of \cref{eqn:rank_probe} are measured against a held-out batch with the validated LQA estimate. \cref{fig:adamw_probing} shows the same overall structure as the Muon trajectory. The loss-optimal step rises by more than an order of magnitude from the volatile head into a mostly flat, far more tolerant bulk, and the deep tail again contributes only a modest residual share of the cumulative loss reduction.

The one systematic difference appears at the head and is expected. The whitened update of Muon drives every singular direction at the same shared global scale, so the head sits at the Edge-of-Stability and its measured optimal step is pinned at that scale (the middle panel of \cref{fig:lr_analysis}). The AdamW update, in contrast, is not normalised. The exact update norm realised along the trajectory varies, and the head's loss-optimal step size inherits this variation, so it varies relatively more freely across checkpoints. As a consequence, the optimal step size profile is more variable across checkpoints than along the Muon trajectory (at early checkpoints it even rises across the bulk before the sharp near-head decline), although the overall trend of a volatile head and a more tolerant bulk still holds. This close structural agreement shows that the measured anisotropic profile persists across the evaluated Muon and AdamW trajectories, rather than being specific to the optimiser that generated the trajectory.

\begin{figure}[ht]
    \centering
    \includegraphics[width=\textwidth]{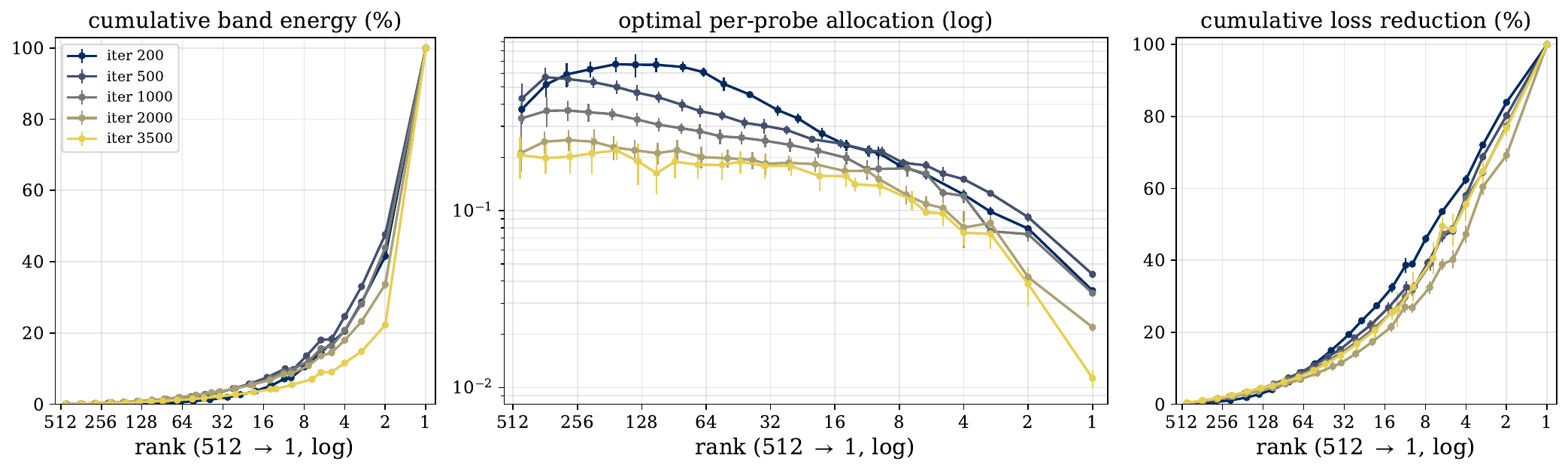}
    \caption{\textbf{Out-of-sample loss-optimal step profile along an AdamW trajectory} (same 64M-parameter ``modded-nanogpt'' model as \cref{fig:lr_analysis}, checkpoints from iterations 200--3500). Out-of-sample spectral probing (\cref{sec:linesearch_analysis}) of the updated first-moment (momentum) buffer, with each per-rank probe's optimal step size measured on a disjoint out-of-sample batch; the horizontal axis runs from the deep tail (rank 512) to the head (rank 1, log scale). \textbf{Left:} cumulative spectral energy, again concentrated in the leading ranks. \textbf{Middle:} per-rank optimal spectral allocation $\eta^*_d$ (the loss-optimal step along each probe; LQA estimate; log axis). As with Muon, the bulk is mostly flat and the optimal allocation drops sharply near the head, which tolerates a far smaller step. Because the AdamW update is unnormalised, the realised update norm varies, and the head's loss-optimal step therefore varies more freely across checkpoints; the slower convergence of AdamW also prolongs the burn-in, giving a more unstable profile in early checkpoints (iterations 200--500) that settles to the same volatile-head shape later (iterations 1000--3500). \textbf{Right:} cumulative loss reduction (LQA estimate): as for Muon, the leading ranks account for most of the achievable reduction and the deep tail contributes a modest residual share. The close structural agreement with \cref{fig:lr_analysis} shows that the volatile-head and tolerant-bulk structure persists across the evaluated Muon and AdamW trajectories.}
    \label{fig:adamw_probing}
\end{figure}

\section{Spectral Probing at Width 768}
\label{app:w768_probing}

The probing analyses of \cref{sec:linesearch_analysis} were run on a width-512 (64M) model. To check that the measured structure is not specific to that scale, the complete out-of-sample probing analysis is repeated on a Muon trajectory of the 124M (width-768) model, matching the smallest trained scale of \cref{sec:experiments} (probing-run setup in \cref{app:experimental_setup}). \cref{fig:w768_probing} gives the same three-panel analysis as \cref{fig:lr_analysis}; the near-head decline of the optimal step size with respect to the log rank axis is shown in the right panel of \cref{fig:lgl_shape} in the main text.

The qualitative observations in \cref{sec:linesearch_analysis} persist at width 768. The profile remains highly anisotropic and stable across checkpoints, with a volatile head near the applied Muon scale, a mostly flat tolerant bulk, and an approximately log-rank-linear near-head transition.

The most relevant visible scale effect is a modest upward shift in the approximate plateau-onset rank, from about 32 at width 512 to about 40 at width 768 (the right panel of \cref{fig:lgl_shape}). Because the transition is gradual, these values should be interpreted as approximate rather than sharply identified thresholds. This shift is consistent with the square-root width heuristic used in \cref{sec:method}, which predicts $k \approx 39$ at width 768. With only two directly probed widths, however, we do not treat this agreement as evidence establishing a general scaling law.

\begin{figure}[ht]
    \centering
    \includegraphics[width=\textwidth]{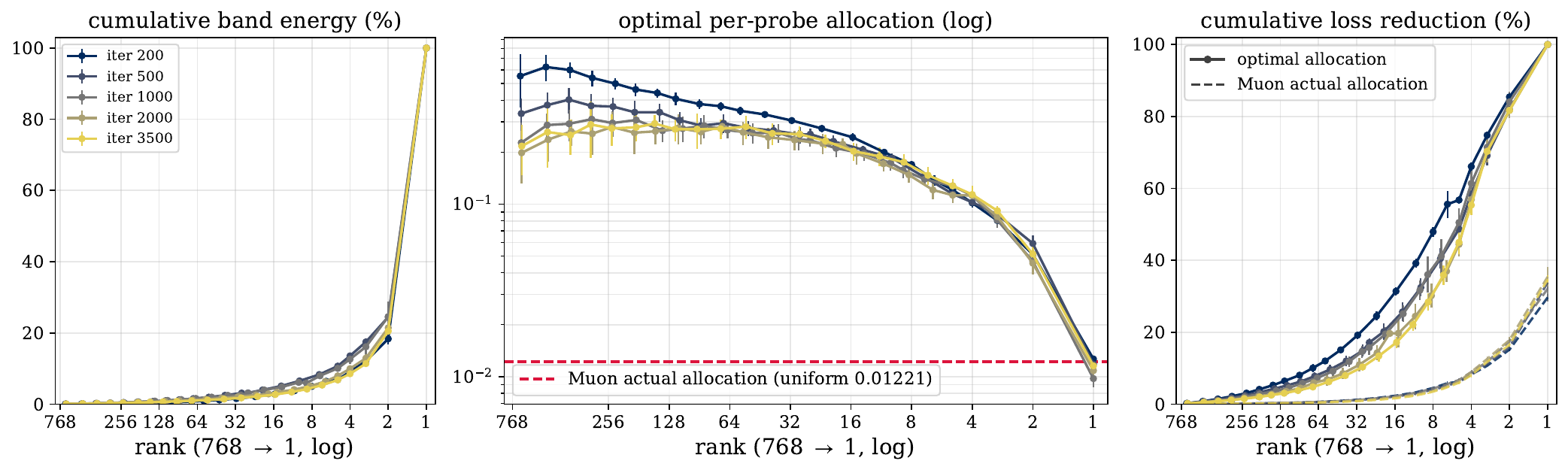}
    \caption{\textbf{Out-of-sample loss-optimal step profile at width 768} (124M model, Muon trajectory). Panels, axes, and protocol as in \cref{fig:lr_analysis}. The qualitative structure of the width-512 analysis is reproduced.}
    \label{fig:w768_probing}
\end{figure}

\section{In-Sample Spectral Probing}
\label{app:insample_probing}

The out-of-sample probing of \cref{sec:linesearch_analysis} measures the step size that reduces loss the most on \emph{unseen} data. For completeness, we also report the \emph{in-sample} variant, in which each probe's optimal step size is measured on the \emph{same} batch whose gradient formed the buffer ($\mathcal{B}_{\mathrm{o}} = \mathcal{B}_{\mathrm{g}}$). This is the quantity that earlier curvature analyses implicitly probe. We report it in the same rank-organised, three-panel layout as the main-text out-of-sample figure (\cref{fig:lr_analysis}), so the two regimes can be compared panel for panel.

\textbf{Observation.} \cref{fig:insample_probing} reports the in-sample profile along the same Muon trajectory, in the same rank-organised, three-panel layout as the out-of-sample figures. The qualitative structure matches the out-of-sample analysis (a volatile head pinned at the Edge-of-Stability and a far more tolerant bulk), but two quantitative differences appear. First, the in-sample optimal allocation declines \emph{monotonically} from the deep tail to the head rather than holding a mostly flat bulk, and the anisotropy spans two to three orders of magnitude, wider than the out-of-sample gap. Second, the cumulative loss reduction under the optimal allocation accrues steadily from the tail, so the bulk and tail appear far more productive than out-of-sample, where the reduction concentrates at the head. Both effects are consistent with in-sample probing partly measuring the optimiser's ability to refit the gradient batch that formed the update; the out-of-sample analysis in the main text removes this confounder and yields the more conservative, generalisation-relevant profile that motivates SAMuon. The per-module sign alignment of \cref{app:sign_alignment} supports the refitting explanation of the in-sample behaviour: tail directions align poorly with the gradients of unseen batches, so their large in-sample tolerance mainly reflects refitting of the sampled batch. The same in-sample inflation of the bulk and tail also occurs along an AdamW trajectory (\cref{fig:insample_probing_adam}), confirming that this batch-refitting confounder is a property of the probe, not of the optimiser that generated the trajectory.

\begin{figure}[ht]
    \centering
    \includegraphics[width=\textwidth]{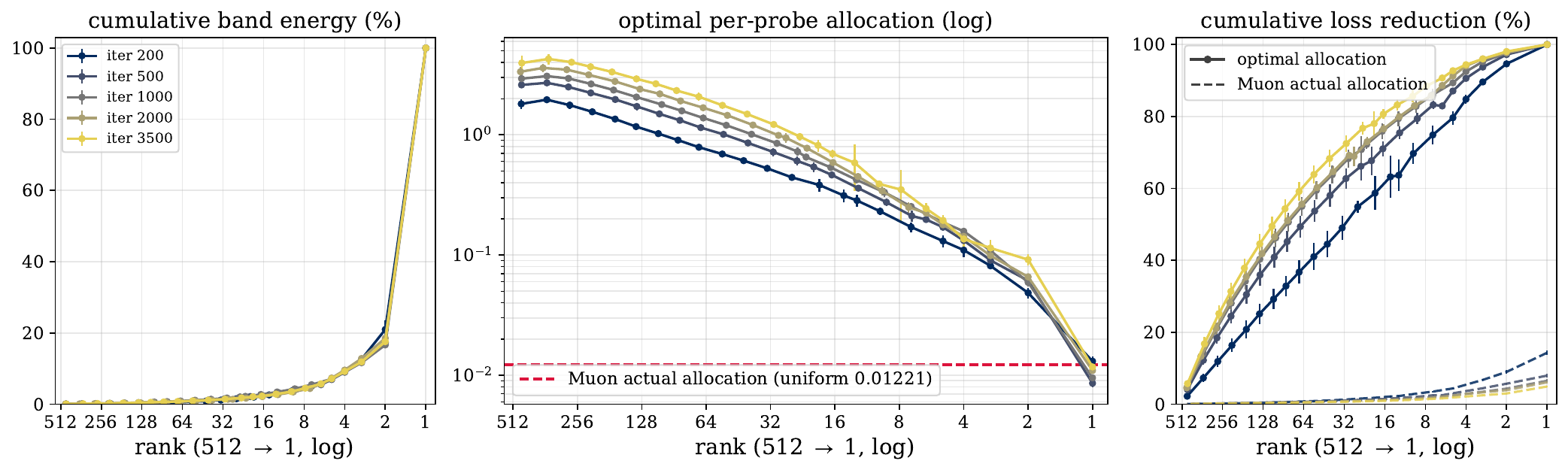}
    \caption{\textbf{In-sample loss-optimal step profile along a Muon trajectory}, in the same panel layout as the out-of-sample \cref{fig:lr_analysis} (same 64M-parameter ``modded-nanogpt'' model and checkpoints), but with each probe's optimal step size measured on the \emph{same} batch whose gradient formed the buffer ($\mathcal{B}_{\mathrm{o}} = \mathcal{B}_{\mathrm{g}}$). Two differences from the out-of-sample profile stand out. First, the entire bulk is raised well above its out-of-sample level (middle panel), widening the anisotropy. Second, the deep tail contributes far more to the idealised cumulative loss reduction (right panel). Both indicate that, in-sample, the bulk mainly yields refitting of the batch that formed the update, a gain that does not transfer to unseen data, which is exactly the confounder the main-text out-of-sample protocol removes.}
    \label{fig:insample_probing}
\end{figure}

\cref{fig:insample_vs_heldout} overlays the two regimes directly along the same Muon trajectory. The out-of-sample optimal allocation (solid) is uniformly more conservative than the in-sample one (dashed) across the entire bulk, with the gap widest in the deep tail, where the in-sample allocation runs several times larger. The per-probe loss-reduction panel tells the same story from the capacity side. For the in-sample case, every probe across the spectrum yields a large reduction; out-of-sample, the tail's contribution collapses and the reduction climbs only towards the head. The in-sample tail gain that does not survive out-of-sample is refitting of the gradient batch, while the residual that does survive is the generalisation-relevant headroom that SAMuon targets. Note that the same in-sample-versus-out-of-sample gap holds along an AdamW trajectory (\cref{fig:insample_vs_heldout_adam}), so this separation is generalisable rather than specific to Muon. Curvature-based accounts such as \citet{MuonCurvature} implicitly measure the in-sample quantity, which our out-of-sample protocol deliberately discounts.

\begin{figure}[ht]
    \centering
    \includegraphics[width=\textwidth]{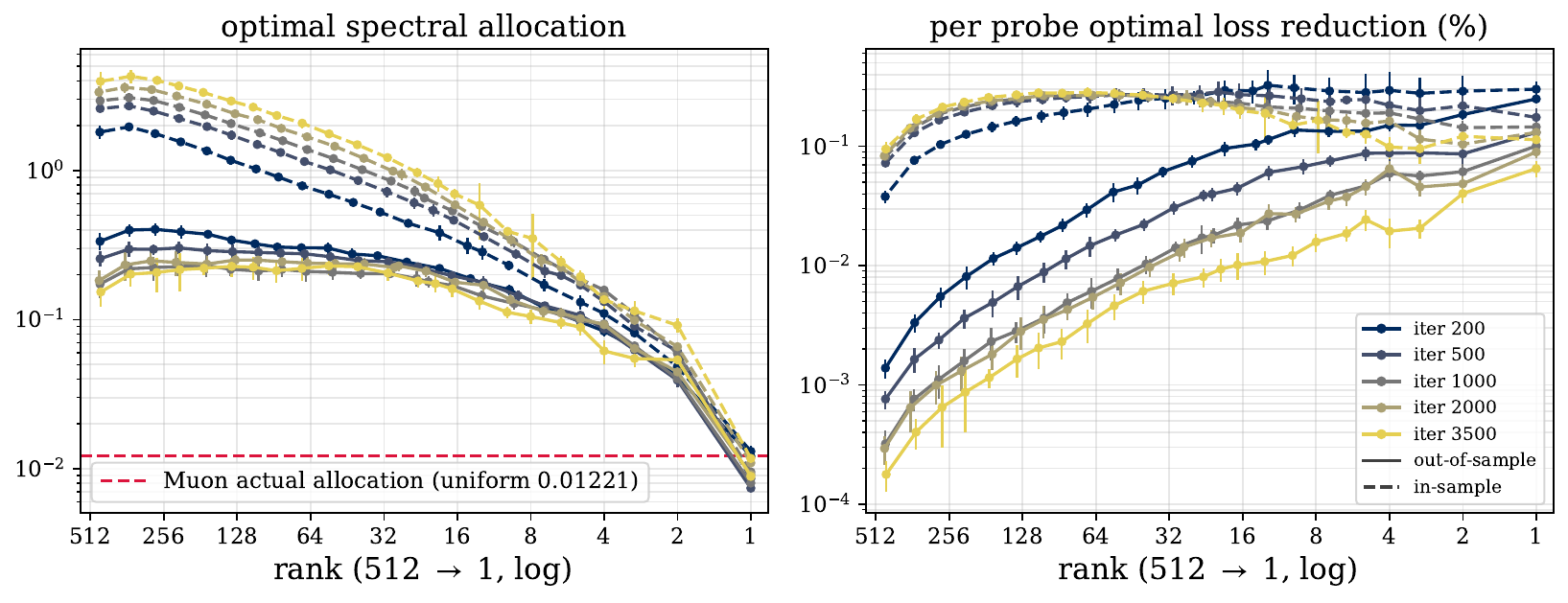}
    \caption{\textbf{In-sample versus out-of-sample spectral profiles along a Muon trajectory} (64M-parameter ``modded-nanogpt'', checkpoints from iterations 200--3500). Solid curves are the out-of-sample (main-text) measurement; dashed curves are the in-sample variant. \textbf{Left:} per-rank optimal spectral allocation. The in-sample allocation is systematically higher across the entire bulk, several times the out-of-sample level. Both collapse together at the head, pinned at the actual Muon scale (red dashed, $0.01221$). \textbf{Right:} per-probe optimal loss reduction (log axis). In-sample, every probe across the spectrum, particularly the deep tail, yields a large reduction; out-of-sample, the per-probe reduction in the tail is far smaller and only climbs towards the head. The gap between the two reflects that the tail spectrum of the gradient batch mainly contributes to overfitting that batch, with limited transferable capacity to out-of-sample batches.}
    \label{fig:insample_vs_heldout}
\end{figure}

\begin{figure}[ht]
    \centering
    \includegraphics[width=\textwidth]{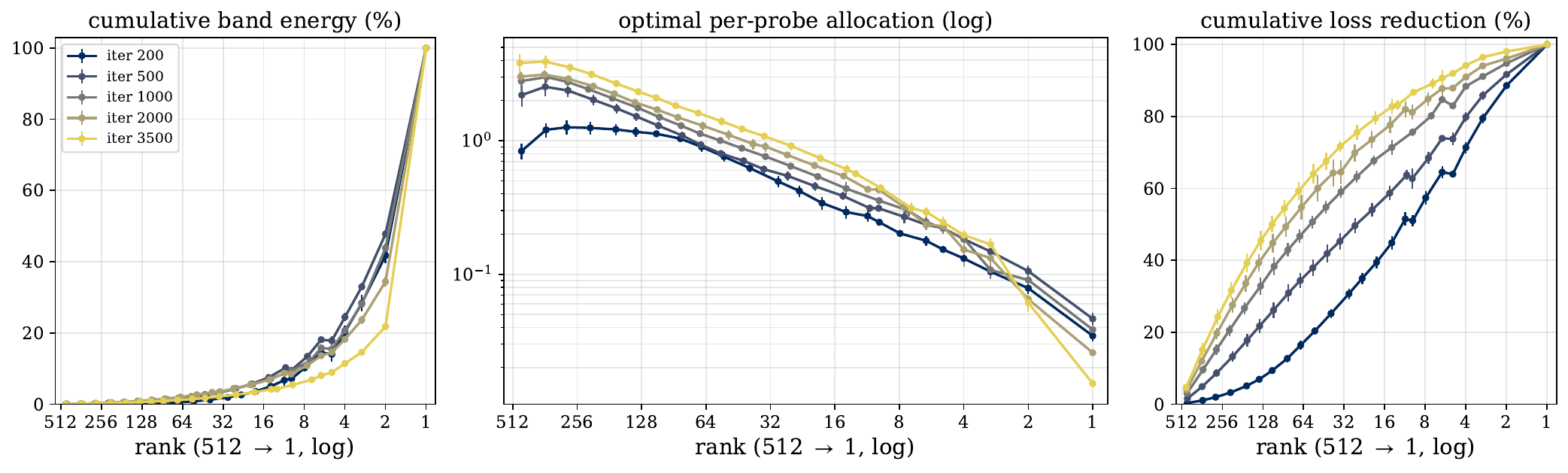}
    \caption{\textbf{In-sample loss-optimal step profile along an AdamW trajectory}, the AdamW analogue of the Muon \cref{fig:insample_probing} and the in-sample counterpart of the out-of-sample \cref{fig:adamw_probing} (same 64M-parameter modded-nanogpt model and checkpoints; three-panel layout as \cref{fig:lr_analysis}), with each probe's optimal step size measured on the \emph{same} batch whose gradient formed the buffer ($\mathcal{B}_{\mathrm{o}} = \mathcal{B}_{\mathrm{g}}$). As for Muon (\cref{fig:insample_probing}), the bulk is lifted above its out-of-sample level (middle panel) and the deep tail contributes far more of the idealised cumulative loss reduction (right panel), the same batch-refitting signatures that the main-text out-of-sample protocol removes.}
    \label{fig:insample_probing_adam}
\end{figure}

\begin{figure}[ht]
    \centering
    \includegraphics[width=\textwidth]{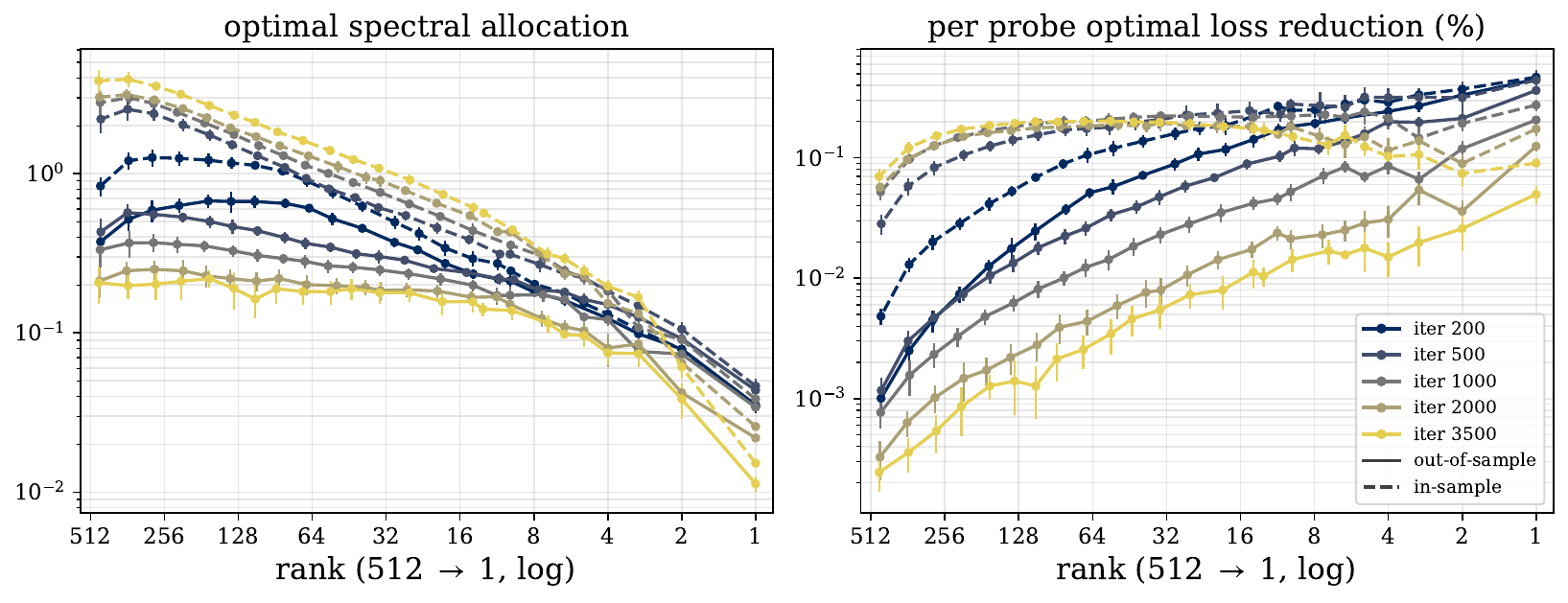}
    \caption{\textbf{In-sample versus out-of-sample spectral profiles along an AdamW trajectory} (64M-parameter ``modded-nanogpt'', checkpoints from iterations 200--3500), the AdamW analogue of the Muon \cref{fig:insample_vs_heldout}; panels and axes as there, with solid curves out-of-sample and dashed in-sample. As for Muon, the in-sample allocation is systematically more permissive across the bulk (widest in the deep tail) and every in-sample probe yields a large loss reduction, whereas out-of-sample the tail's contribution collapses; the gap is the batch-refitting capacity that does not transfer to unseen data.}
    \label{fig:insample_vs_heldout_adam}
\end{figure}

\section{Construction of the Spectral Allocation Comparison}
\label{app:profile_cmp_details}

This section details the construction of the curves in \cref{fig:profile_cmp}, which compares the spectral scaling each optimiser applies against the measured loss-optimal step profile. All curves are computed at checkpoint 1000 of the out-of-sample probing analyses (\cref{sec:linesearch_analysis}), averaged across out-of-sample batches, and head-normalised so that the top singular direction (rank 1, largest $\sigma$) sits at $1$. A single checkpoint is shown for visual clarity. The left panel plots the spectral allocation against spectral rank (log axes). The right panel plots the same allocation against singular value $\sigma$ (linear axes), making explicit the transfer function each optimiser applies to the spectrum.
\begin{itemize}[leftmargin=2em]
    \item \emph{Loss-optimal step profile}: the per-rank loss-optimal step size measured by out-of-sample spectral probing (the measured profile of the Muon trajectory at checkpoint 1000 in \cref{fig:lr_analysis}).
    \item \emph{SGD}: a virtual update that assigns each rank a scale proportional to its average singular value, that is, the allocation produced when SGD applies the momentum buffer directly.
    \item \emph{Adam}: the true Adam update $\mathbf{\Delta}^{\mathrm{A}}$ projected onto each rank of the singular basis of its own buffer, $a_i = \bm u_i^\top \mathbf{\Delta}^{\mathrm{A}} \bm v_i$, that is, the diagonal of $\mathbf{U}^\top \mathbf{\Delta}^{\mathrm{A}} \mathbf{V}$. Only these diagonal terms are tracked. The off-diagonal terms $\bm u_i^\top \mathbf{\Delta}^{\mathrm{A}} \bm v_j$ with $i \neq j$, which describe how Adam couples distinct singular directions, are ignored. This update is sampled from checkpoint 1000 of the AdamW trajectory.
    \item \emph{Muon}: exact whitening, a uniform scale across the spectrum.
    \item \emph{\samuonlite{}}: the two-level allocation of \cref{sec:method} with $k = 1$, holding the head at the same scale as Muon and amplifying the bulk uniformly to $\gamma = 5$, the measured head gap.
    \item \emph{\samuon{}}: the log-rank-linear allocation of \cref{eqn:lgl_profile} at $\gamma = 7.07$, its tuned optimum at all batch sizes (\cref{tab:optim_params}), ramping from the anchored head to the flat tail boost over the leading $k$ ranks.
\end{itemize}
Both the \samuon{} and \samuonlite{} allocations are drawn at their post-warmup target profiles. The comparison summarises the argument of \cref{sec:linesearch_analysis} that the SGD allocation is anti-correlated with the measured profile; Adam dampens but does not overcome the head's dominance; Muon reallocates several times the scale towards the bulk relative to SGD; \samuonlite{} holds the head at the same scale as Muon and lifts the bulk uniformly towards the measured profile; and \samuon{} follows the measured near-head decline itself, thereby providing the closest alignment among the five methods.

\section{Convergence proof for the idealised SAMuon update}
\label{app:convergence_samuon}
In this section, we establish convergence guarantees for the idealised exact-whitening form of the \samuon{} / \samuonlite{} variants of \cref{alg:samuon}. By extending the convergence analysis framework of standard Muon (see Theorem 4.1 in \citet{MuonConvergence}) to accommodate the head-anchored spectral allocation of \cref{subsec:spectral_scaling}, we prove convergence to an $\epsilon$-nuclear-norm stationary point. Although the spectral shaping differs substantially from uniform whitening, the analysis depends on it only through its bounds, so a single argument covers \samuon{}, \samuonlite{}, and other related profiles while mathematically formalising the trade-off introduced by the tail boost $\gamma$.

\subsection{Setup and notation}
Following the standard matrix-structured analysis of Muon, we state the result for one two-dimensional parameter matrix, matching the blockwise form of the optimiser update. Accordingly, the theorem below is a single-matrix guarantee; SAMuon applies the same update independently to every eligible matrix in the model. Consider the task of minimising a non-convex objective function $l: \mathbb{R}^{m \times n} \rightarrow \mathbb{R}$. Let $\mathbf{W}_t \in \mathbb{R}^{m \times n}$ denote the parameter matrix at iteration $t$, and let $r = \min(m, n)$ denote the maximum possible rank of the matrix. 

The optimiser maintains an exponential moving average (EMA) of the stochastic gradients, initialised at zero as in \cref{alg:samuon}, and reshapes its spectrum with the head-anchored allocation of \cref{subsec:spectral_scaling}.

For $t=1,\ldots,T$, the update rule is
\begin{align}
    \mathbf{M}_0 &= \mathbf{0}, \\
    \mathbf{M}_t &= \mu \mathbf{M}_{t-1} + (1 - \mu) \mathbf{G}_t \\
    \mathbf{Q}_t &= \mathrm{SAMuon}(\mathbf{M}_t) \\
    \mathbf{W}_{t+1} &= \mathbf{W}_t - \eta \mathbf{Q}_t
\end{align}
where $\eta$ is a constant effective step size, $\mu \in [0, 1)$ is the momentum parameter, and $\mathbf{G}_t$ is the stochastic mini-batch gradient. Because this is a single-matrix result, the fixed positive Scion scale $\kappa$ is absorbed into $\eta$. As in \cref{sec:preliminaries}, each momentum buffer is assumed to have rank $r$. Here $\mathrm{SAMuon}(\cdot)$ denotes the anisotropic spectral scaling of \cref{eqn:samuon_update}. Following \citet{MuonConvergence}, we idealise the Newton-Schulz iteration as exact whitening. Then each singular direction receives exactly its prescribed scale, with the head remaining at the same unit scale as in Muon, while $\mathbf{Q}_t$ shares the singular vectors of $\mathbf{M}_t$. The guarantee therefore covers this idealised exact-whitening update, as is standard for analyses of this type; the finite-step Newton-Schulz, randomised-SVD, and power-iteration approximations run by the deployed algorithms lie outside the bound.

Consequently, if the Singular Value Decomposition (SVD) of the momentum buffer is $\mathbf{M}_t = \mathbf{U} \mathbf{\Sigma} \mathbf{V}^\top$, the preconditioned update evaluates to $\mathbf{Q}_t = \mathbf{U} \hat{\mathbf{\Sigma}} \mathbf{V}^\top$ with modified singular values
\begin{equation}
    \hat{\sigma}_i =
    \begin{cases}
        s_i(t) & i \le k \quad\text{(the head region, with the head anchored at $\hat\sigma_1 = s_1(t) = 1$)} \\[4pt]
        \gamma_t & i > k \quad\text{(the boosted tail),}
    \end{cases}
\end{equation}
following the head-anchored allocation of \cref{eqn:samuon_update}; the two-level \samuonlite{} is the case $k = 1$, and \samuon{} takes the log-rank-linear scales of \cref{eqn:lgl_profile}. The analysis below depends on this scaling only through its bounds, with the amplification fixed at $\gamma \ge 1$. As the warmup raises the whole profile monotonically from the all-ones Muon scaling towards its target over the first $T_0$ steps (so $1 \le s_i(t) \le \gamma_t \le \gamma$ throughout), all modified singular values satisfy
\begin{equation}
\label{eq:spectral_bounds}
    1 \;\le\; \hat{\sigma}_i \;\le\; \gamma \qquad \text{for all } i \text{ and } t.
\end{equation}
All bounds below are therefore stated in terms of the post-warmup amplification $\gamma$, and hold uniformly over training.

\subsection{Assumptions}
We rely on the following standard assumptions regarding the loss landscape and the stochastic oracle, following \citet{MuonConvergence}. Both landscape assumptions are idealisations relative to Transformer practice: a single global smoothness constant is assumed and the variance bound is assumed uniform along the trajectory.

\begin{assumption}[$L$-Smoothness]
\label{assum:smoothness}
The objective function $l$ has $L$-Lipschitz continuous gradients with respect to the Frobenius norm. For any $\mathbf{W}, \mathbf{W}' \in \mathbb{R}^{m \times n}$:
\begin{equation}
    l(\mathbf{W}) \le l(\mathbf{W}') + \langle \nabla l(\mathbf{W}'), \mathbf{W} - \mathbf{W}' \rangle + \frac{L}{2} \|\mathbf{W} - \mathbf{W}'\|_F^2
\end{equation}
\end{assumption}

\begin{assumption}[Bounded Variance]
\label{assum:variance}
Let $\mathcal{F}_t$ denote the optimisation history before $\mathbf{G}_t$ is sampled. At each step $t$, the stochastic oracle provides a conditionally unbiased gradient estimator with bounded conditional variance:
\begin{align}
    \mathbb{E}[\mathbf{G}_t \mid \mathcal{F}_t] &= \nabla l(\mathbf{W}_t), \\
    \mathbb{E}[\|\mathbf{G}_t - \nabla l(\mathbf{W}_t)\|_F^2 \mid \mathcal{F}_t] &\le \frac{\sigma^2}{B},
\end{align}
where $B$ is the mini-batch size.
\end{assumption}

\begin{assumption}[Lower-Bounded Objective]
\label{assum:bounded_obj}
The objective function $l(\mathbf{W})$ is bounded from below by a global minimum $l^*$.
\end{assumption}

\subsection{Main convergence theorems}

\begin{theorem}[Ergodic Convergence of the Idealised SAMuon Update]
\label{thm:samuon_convergence}
Under \cref{assum:smoothness,assum:variance,assum:bounded_obj}, using the idealised exact-whitening SAMuon update with constant learning rate $\eta$ for $T$ iterations yields the following bound on the expected nuclear norm of the true gradient:
\begin{align}
    \frac{1}{T} \sum_{t=1}^T \mathbb{E}[\|\nabla l(\mathbf{W}_t)\|_*] \le &\frac{l(\mathbf{W}_1) - l^*}{\eta T} + \frac{L r \gamma^2 \eta}{2} \nonumber \\
    &+ \sqrt{r}(1 + \gamma) \left( \frac{\sigma \sqrt{1-\mu}}{\sqrt{B(1+\mu)}} + \frac{\mu\|\nabla l(\mathbf{W}_1)\|_F}{(1-\mu)T} + \frac{\sqrt{r} \gamma \eta \mu L}{1-\mu} \right).
\end{align}
\end{theorem}

The following choices recover the stochastic and deterministic convergence orders of Muon while retaining the same zero-buffer initialisation as \cref{alg:samuon}.

\begin{corollary}[Stochastic and Deterministic Convergence Rates]
\label{cor:explicit_rate}
Let $\mathcal{D}=l(\mathbf{W}_1)-l^*$ and regard $\gamma$ as fixed.
\begin{enumerate}[label=(\roman*)]
    \item In the stochastic setting, set $B=1$, $\eta=\sqrt{(1-\mu)\mathcal{D}/(rTL)}$, and $1-\mu=\min\{\sqrt{L\mathcal{D}}/(\sigma\sqrt{T}),1\}$. Then
    \begin{equation}
        \frac{1}{T}\sum_{t=1}^T \mathbb{E}[\|\nabla l(\mathbf{W}_t)\|_*]
        \le \mathcal{O}_{\gamma}\!\left(
        \sqrt[4]{\frac{r^2L\mathcal{D}\sigma^2}{T}}
        + \sqrt{\frac{rL\mathcal{D}}{T}}
        + \frac{\sqrt{r}\,\sigma}{\sqrt{T}}
        \right).
    \end{equation}
    \item In the deterministic setting ($\sigma=0$), set $\eta=\sqrt{\mathcal{D}/(rTL)}$ and hold $\mu<1$ fixed independently of $T$. Then
    \begin{equation}
        \frac{1}{T}\sum_{t=1}^T \|\nabla l(\mathbf{W}_t)\|_*
        \le \mathcal{O}_{\gamma}\!\left(\sqrt{\frac{rL\mathcal{D}}{T}}\right).
    \end{equation}
\end{enumerate}
The bounds follow by substituting these choices into \cref{thm:samuon_convergence} and using the standard consequence $\|\nabla l(\mathbf{W}_1)\|_F\le\sqrt{2L\mathcal{D}}$ of $L$-smoothness and lower boundedness; the zero-initialisation term is lower order in both cases. Here $\mathcal{O}_{\gamma}$ suppresses constants depending on the fixed amplification bound $\gamma$ (and, in the deterministic case, on the fixed momentum coefficient).
\end{corollary}

\begin{remark}
Setting $\gamma = 1$ (and hence $\gamma_t \equiv 1$, $\hat{\sigma}_i \equiv 1$) specialises \cref{thm:samuon_convergence} to Muon with the zero-buffer initialisation of \cref{alg:samuon}. This initialisation produces the vanishing initial-gradient term in the theorem; the first-gradient initialisation analysed by \citet{MuonConvergence} instead produces a vanishing initial-noise term. Both conventions give the stochastic and deterministic convergence orders stated in \cref{cor:explicit_rate}.

Comparing against this Muon baseline makes the trade-off of the tail boost explicit: while $\gamma$ accelerates descent across the spectrum, it linearly amplifies the stochastic noise floor ($\frac{\sigma}{\sqrt{B}}$) and quadratically amplifies the deterministic penalty. Of our design choices, the head anchor ($\hat{\sigma}_1 = 1$) rests on the empirical Edge-of-Stability measurement of \cref{sec:linesearch_analysis} rather than on the bound, which is indifferent to which rank carries which scale. Furthermore, the warmup schedule withholds the boost during early training, precisely when the momentum tracking error $\|\mathbf{\Delta}_t\|_F$ that the amplification multiplies is at its largest. The bound itself is shape-agnostic within $[1, \gamma]$: the log-rank-linear profile of \samuon{} moderates the scales of the leading ranks, so it satisfies the same guarantee at the same $\gamma$ while pushing only the tail to the peak scale.
\end{remark}

\subsection{Detailed proofs}

\begin{proof}[Proof of \cref{thm:samuon_convergence}]
\textbf{Step 1: The Descent Lemma and the Anisotropic Penalty}\\
We begin by substituting the SAMuon parameter update $\mathbf{W}_{t+1} - \mathbf{W}_t = -\eta \mathbf{Q}_t$ into the $L$-smoothness inequality (\cref{assum:smoothness}):
\begin{equation}
\label{eq:descent_initial}
    l(\mathbf{W}_{t+1}) \le l(\mathbf{W}_t) - \eta \langle \nabla l(\mathbf{W}_t), \mathbf{Q}_t \rangle + \frac{L \eta^2}{2} \|\mathbf{Q}_t\|_F^2
\end{equation}
Because $\mathbf{Q}_t$ is constrained by the spectral bounds of \cref{eq:spectral_bounds}, its singular values are bounded above by $\gamma$. The squared Frobenius norm is geometrically bounded independently of the raw gradient scale:
\begin{equation}
\label{eq:f_norm_bound}
    \|\mathbf{Q}_t\|_F^2 = \sum_{i=1}^r \hat{\sigma}_i^2 \le \sum_{i=1}^r \gamma^2 = r \gamma^2
\end{equation}

\textbf{Step 2: Decomposing the Descent Direction}\\
We evaluate the inner product between the true gradient and the anisotropic update by adding and subtracting the momentum buffer $\mathbf{M}_t$:
\begin{equation}
\label{eq:inner_prod_decomp}
    \langle \nabla l(\mathbf{W}_t), \mathbf{Q}_t \rangle = \langle \mathbf{M}_t, \mathbf{Q}_t \rangle + \langle \nabla l(\mathbf{W}_t) - \mathbf{M}_t, \mathbf{Q}_t \rangle
\end{equation}

\textbf{Step 3: The Nuclear Norm Alignment}\\
Since $\mathbf{M}_t$ and $\mathbf{Q}_t$ share identical singular vectors by design, their trace inner product collapses into a weighted sum of the singular values of $\mathbf{M}_t$. Because \cref{eq:spectral_bounds} guarantees $\hat{\sigma}_i \ge 1$ for all dimensions, we obtain:
\begin{equation}
\label{eq:nuclear_norm_align}
    \langle \mathbf{M}_t, \mathbf{Q}_t \rangle = \sum_{i=1}^r \sigma_i \hat{\sigma}_i \ge \sum_{i=1}^r \sigma_i = \|\mathbf{M}_t\|_*
\end{equation}

\textbf{Step 4: Bounding the Amplified Tracking Error}\\
Let $\mathbf{\Delta}_t = \nabla l(\mathbf{W}_t) - \mathbf{M}_t$ represent the tracking error. Applying the Cauchy-Schwarz inequality, followed by the upper bound $\|\mathbf{Q}_t\|_F \le \sqrt{r} \gamma$ derived from \cref{eq:f_norm_bound} gives:
\begin{equation}
\label{eq:tracking_error_bound}
    \langle \mathbf{\Delta}_t, \mathbf{Q}_t \rangle \ge -\|\mathbf{\Delta}_t\|_F \|\mathbf{Q}_t\|_F \ge -\sqrt{r} \gamma \|\mathbf{\Delta}_t\|_F
\end{equation}

Substituting \cref{eq:nuclear_norm_align,eq:tracking_error_bound} into \cref{eq:inner_prod_decomp} yields:
\begin{equation}
    \langle \nabla l(\mathbf{W}_t), \mathbf{Q}_t \rangle \ge \|\mathbf{M}_t\|_* - \sqrt{r} \gamma \|\mathbf{\Delta}_t\|_F
\end{equation}

Relating $\|\mathbf{M}_t\|_*$ to the true gradient via the reverse triangle inequality ($\|\mathbf{A}\|_* \ge \|\mathbf{B}\|_* - \|\mathbf{A}-\mathbf{B}\|_*$) and the matrix norm bound $\|\mathbf{\Delta}_t\|_* \le \sqrt{r}\|\mathbf{\Delta}_t\|_F$ gives:
\begin{align}
    \|\mathbf{M}_t\|_* &\ge \|\nabla l(\mathbf{W}_t)\|_* - \sqrt{r} \|\mathbf{\Delta}_t\|_F
\end{align}

Combining these establishes a lower bound for the descent direction, exposing the $\gamma$ tracking multiplier:
\begin{equation}
\label{eq:final_descent_direction}
    \langle \nabla l(\mathbf{W}_t), \mathbf{Q}_t \rangle \ge \|\nabla l(\mathbf{W}_t)\|_* - \sqrt{r}(1 + \gamma) \|\mathbf{\Delta}_t\|_F
\end{equation}

\textbf{Step 5: Bounding the EMA Tracking Error}\\
Define the moving average of the true gradients by $\mathbf{C}_0=\mathbf{0}$ and
\begin{equation}
    \mathbf{C}_t=\mu\mathbf{C}_{t-1}+(1-\mu)\nabla l(\mathbf{W}_t).
\end{equation}
The tracking error decomposes as
\begin{equation}
    \|\mathbf{\Delta}_t\|_F
    \le \|\mathbf{M}_t-\mathbf{C}_t\|_F
    + \|\nabla l(\mathbf{W}_t)-\mathbf{C}_t\|_F.
\end{equation}
For the first term, unrolling the two moving averages gives
\begin{equation}
    \mathbf{M}_t-\mathbf{C}_t
    =(1-\mu)\sum_{i=1}^{t}\mu^{t-i}\big(\mathbf{G}_i-\nabla l(\mathbf{W}_i)\big).
\end{equation}
The conditional unbiasedness in \cref{assum:variance} makes the summands a martingale-difference sequence. Hence the cross terms vanish after taking expectations, and Cauchy--Schwarz gives
\begin{align}
    \mathbb{E}[\|\mathbf{M}_t-\mathbf{C}_t\|_F]
    &\le \sqrt{\mathbb{E}[\|\mathbf{M}_t-\mathbf{C}_t\|_F^2]} \\
    &\le \frac{\sigma}{\sqrt{B}}(1-\mu)\sqrt{\sum_{i=1}^{t}\mu^{2(t-i)}}
    \le \frac{\sigma\sqrt{1-\mu}}{\sqrt{B(1+\mu)}}.
\end{align}
For the second term, let $\mathbf{D}_t=\nabla l(\mathbf{W}_t)-\mathbf{C}_t$. Since $\mathbf{C}_0=\mathbf{0}$, $\mathbf{D}_1=\mu\nabla l(\mathbf{W}_1)$. For $t\ge2$, $L$-smoothness and \cref{eq:f_norm_bound} give
\begin{align}
    \|\mathbf{D}_t\|_F
    &\le \mu\|\mathbf{D}_{t-1}\|_F
      +\mu\|\nabla l(\mathbf{W}_t)-\nabla l(\mathbf{W}_{t-1})\|_F \\
    &\le \mu\|\mathbf{D}_{t-1}\|_F+\mu L\eta\sqrt{r}\gamma.
\end{align}
Unrolling this recursion yields
\begin{equation}
    \|\mathbf{D}_t\|_F
    \le \mu^t\|\nabla l(\mathbf{W}_1)\|_F
    +\frac{\sqrt{r}\gamma\eta\mu L}{1-\mu}.
\end{equation}
Combining the two components therefore gives
\begin{equation}
\label{eq:ema_tracking_bound}
    \mathbb{E}[\|\mathbf{\Delta}_t\|_F]
    \le \frac{\sigma\sqrt{1-\mu}}{\sqrt{B(1+\mu)}}
    +\mu^t\|\nabla l(\mathbf{W}_1)\|_F
    +\frac{\sqrt{r}\gamma\eta\mu L}{1-\mu}.
\end{equation}

\textbf{Step 6: Telescoping the Descent Bound}\\
Taking the expectation of the descent lemma (\cref{eq:descent_initial}) and substituting \cref{eq:f_norm_bound,eq:final_descent_direction} gives:
\begin{equation}
    \mathbb{E}[l(\mathbf{W}_{t+1})] \le \mathbb{E}[l(\mathbf{W}_t)] - \eta \mathbb{E}[\|\nabla l(\mathbf{W}_t)\|_*] + \eta\sqrt{r}(1 + \gamma) \mathbb{E}[\|\mathbf{\Delta}_t\|_F] + \frac{L \eta^2 r \gamma^2}{2}
\end{equation}

Rearranging gives the following bound on the expected nuclear norm of the true gradient:
\begin{equation}
    \eta \mathbb{E}[\|\nabla l(\mathbf{W}_t)\|_*] \le \mathbb{E}[l(\mathbf{W}_t)] - \mathbb{E}[l(\mathbf{W}_{t+1})] + \eta\sqrt{r}(1 + \gamma) \mathbb{E}[\|\mathbf{\Delta}_t\|_F] + \frac{L \eta^2 r \gamma^2}{2}
\end{equation}

Summing this inequality from $t=1$ to $T$, noting that $l(\mathbf{W}_{T+1}) \ge l^*$ (\cref{assum:bounded_obj}), and dividing by $\eta T$ yields:
\begin{equation}
    \frac{1}{T} \sum_{t=1}^T \mathbb{E}[\|\nabla l(\mathbf{W}_t)\|_*] \le \frac{l(\mathbf{W}_1) - l^*}{\eta T} + \frac{L r \gamma^2 \eta}{2} + \frac{\sqrt{r}(1 + \gamma)}{T} \sum_{t=1}^T \mathbb{E}[\|\mathbf{\Delta}_t\|_F]
\end{equation}

Substituting \cref{eq:ema_tracking_bound} and using $T^{-1}\sum_{t=1}^{T}\mu^t\le \mu/((1-\mu)T)$ proves \cref{thm:samuon_convergence}.
\end{proof}                                          

\section{Detailed experimental setup}
\label{app:experimental_setup}

In this section, we provide the exact hyperparameters and architectural details to ensure full reproducibility of our results. Our codebase and training pipeline are built directly upon the open-source implementation provided by \citet{scion}.

\subsection{Motivation for using Scion as the ``Muon'' baseline}
\label{app:scion_over_muon}
In our experiments, Muon serves as a crucial strong baseline. Instead of the vanilla Muon implementation, we are in fact using the improved Muon version, Scion \citep{scion}. We select Scion over the vanilla Muon implementation for three strategic reasons: (i) \textit{Efficiency:} Scion replaces the auxiliary Adam optimiser for embeddings with a lightweight signed update rule, reducing memory overhead; (ii) \textit{Stability:} It employs a theoretically grounded normalisation scheme that decouples the learning rate from model scale, facilitating consistent hyperparameter tuning; and (iii) \textit{Performance:} Scion consistently demonstrates improved empirical convergence and large-batch scalability over vanilla Muon.

\subsection{Model architecture}
All models used in this work are decoder-only Transformers based on the GPT-2 architecture, modified with RMSNorm and Rotary Positional Embeddings (RoPE) (refer to \citet{scion} for detailed modifications). The 64M and 124M models used for the spectral probing analyses, and the 124M, 300M, and 1B models used for the training experiments, all share the same depth (12 layers) and head dimension, so width is the only scale axis among them. The vocabulary size is padded to 50,304 to improve kernel efficiency. The complete architectural specifications are summarised in \cref{tab:model_arch}.

\begin{table}[ht]
\caption{\textbf{Model architectures.} The 64M, 124M, 300M and 1B models used for probing and/or training experiments. All utilise a head dimension of 128 and a depth of 12 layers.}
\label{tab:model_arch}
\begin{center}
\begin{small}

\begin{tabular}{lcccc}
\toprule
Specification & \textbf{64M} & \textbf{124M} & \textbf{300M} & \textbf{1B} \\
\midrule
Layers ($L$) & 12 & 12 & 12 & 12 \\
Model Dimension ($d_\text{model}$) & 512 & 768 & 1280 & 2560 \\
Attention Heads ($H$) & 4 & 6 & 10 & 20 \\
Head Dimension ($d_h$) & 128 & 128 & 128 & 128 \\
Vocabulary Size & 50,304 & 50,304 & 50,304 & 50,304 \\
Context Length & 1024 & 1024 & 1024 & 1024 \\
\bottomrule
\end{tabular}

\end{small}
\end{center}
\end{table}

\subsection{Training hyperparameters}

All experiments train on the FineWeb dataset \citep{fineweb} (100B-token sample; ODC-By 1.0 licence) at global batch sizes of 1M, 2M, and 4M tokens (1024, 2048, and 4096 sequences). The 124M and 300M models use a ${\sim}$10B-token budget, with iteration counts of 9600, 4800, and 2400 at the three batch sizes, and the 1B model a ${\sim}$20B-token budget. For the 124M model this budget is ${\sim}4\times$ the Chinchilla-optimal allocation \citep{chinchilla}, falling to ${\sim}1.7\times$ at 300M and ${\sim}1\times$ at 1B.

\textbf{Optimiser Settings.}
All hyperparameters are swept per batch size on the 124M model and transferred across model scales, with the single exception of the AdamW learning rate, which is halved at each scale step ($2^{-8.5} \to 2^{-9.5} \to 2^{-10.5}$ at batch size 1024) following the optimal-learning-rate trend for AdamW reported in \citet{scion}. For AdamW we sweep the learning rate on a half-power-of-two grid around the \citet{scion} value. For Muon we sweep the spectral radius \texttt{radius\_2d} (the learning-rate factor on the 2D, Muon-effective modules) over the $\sqrt{2}$-spaced grid $\{35.36, 50, 70.71, 100, 141.42\}$. Since this sweep is flat around its optimum, both \samuon{} instantiations instead run at the fixed radius $50$, the Scion default, and sweep only the tail boost $\gamma$ over $\{3.54, 5, 7.07, 10, 14.14\}$, a grid centred on the measured head gap. The complete sweeps are reported in \cref{tab:hp_search_124m}, and the per-batch optima are collected in \cref{tab:optim_params}.

\textbf{Probing-run training setup.} The main-text spectral probing analyses are run on the 64M and 124M models, following the \citet{scion} protocol: $5100$ training steps at batch size $512$, using the optimiser hyperparameters specified by \citet{scion}. For the smallest 64M model, the probed Muon trajectory reaches a final validation loss of $3.4292$, and the AdamW trajectory probed in \cref{app:adamw_probing} reaches $3.4913$ under the same protocol. The width-768 (124M) Muon trajectory probed in \cref{app:w768_probing} follows exactly the 124M setup of the same protocol, with the learning rate transferred from the 64M model following \citet{scion}, and reaches a final validation loss of $3.2812$.

\textbf{Profile warmup schedule.} The spectral shaping is itself warmed up. In the full-budget training runs, the \emph{cosine} warmup weight $w_t = \tfrac{1}{2}\big(1 - \cos(\pi \min(t/T_0,\, 1))\big)$ spans the first $30\%$ of training iterations ($T_0 = 0.3\,T$), interpolating the whole profile from the Muon scaling through $\gamma_t = 1 + (\gamma - 1)\, w_t$ and $s_i(t) = 1 + (s_i - 1)\, w_t$, after which it is held at its target. The warmup span mirrors the trapezoidal schedule's final linear-decay phase ($28.5\%$ of training) with a matching ${\sim}30\%$ ramp at the start. This lets the model first settle onto a stable Muon trajectory before the shaping fully engages, and applies identically to both instantiations (for \samuonlite{}, it reduces to the single-scale ramp of $\gamma_t$). The shortened runs used only for the token-efficiency measurement instead retain the fixed warmup horizon $T_0=0.3\,N$ of the corresponding Muon baseline, as detailed in \cref{app:speedup}.

This warmup is not arbitrary since the spectral anisotropy itself \emph{emerges} slowly over early training. \cref{fig:warmup_motivation} tracks the out-of-sample profile along the Muon trajectory and shows that the head gap rises from near-isotropic ($\approx 2$ at iteration 50) towards its final level ($\approx 5$) over the first few hundred iterations, while the head's concentration of spectral energy and directional curvature also climbs at the same time. Applying the full shaping from the start would therefore overdrive a spectrum that is not yet highly anisotropic. Warming the profile instead lets the optimiser track the anisotropy as it emerges and supplies a measurement-grounded justification for delaying the full shaping (the span itself mirrors the schedule's decay phase, as above). The final head gap also anchors the tail-boost \emph{search}. The $\gamma$ grid of our experiments is centred on the measured gap (${\approx}5$), and the optima sit at or moderately above it, namely $\gamma^{*} = 7.07$ for \samuon{} at all batch sizes and between $7.07$ and $10$ for \samuonlite{} (\cref{tab:hp_search_124m}).

\begin{figure}[ht]
    \centering
    \includegraphics[width=\textwidth]{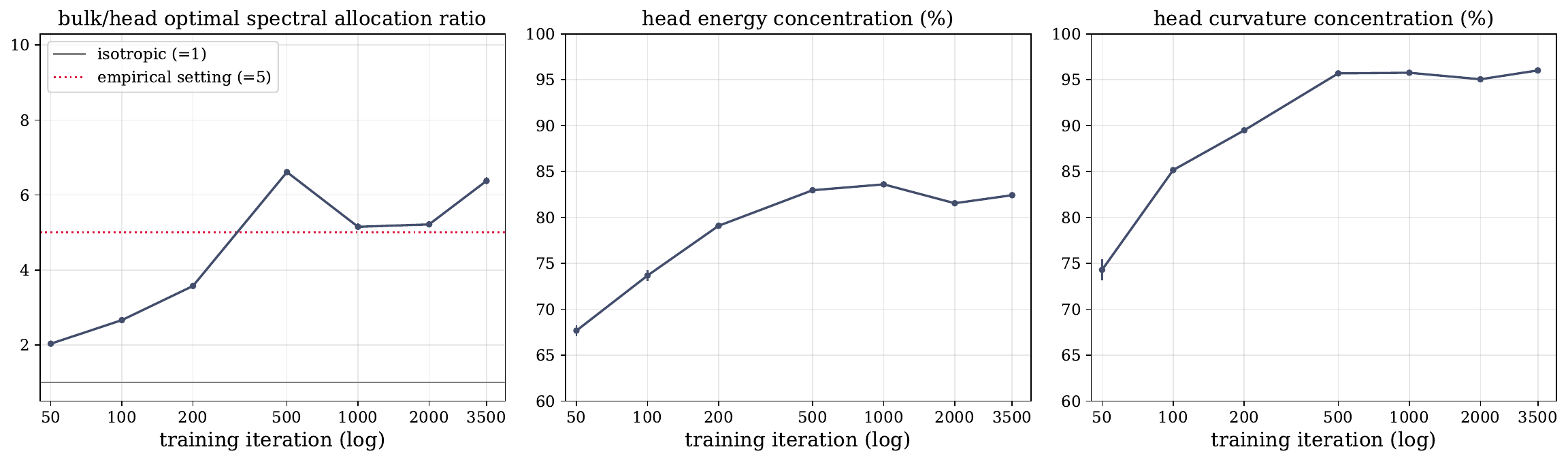}
    \caption{\textbf{The spectral anisotropy emerges over early training, motivating the slow warmup.} Out-of-sample probing along the Muon trajectory (64M-parameter ``modded-nanogpt''), against training iteration (log axis). \textbf{Left:} the head gap (rank-2 over rank-1 optimal spectral allocation ratio, the largest single gap in the measured profile) rises from near-isotropic ($\approx 2$) at iteration 50 towards its final level of $\approx 5$ (red dotted), on which the $\gamma$ search grid of our experiments is centred, within a few hundred iterations, and stays there; the isotropic-curvature assumption ($=1$) is violated throughout and increasingly so. \textbf{Middle:} the head's concentration of spectral energy, $\sigma_1^2 / \sum_i \sigma_i^2$; \textbf{Right:} the head's concentration of directional curvature, $\big(\Delta\bm\theta^{(1)\top}\mathbf{H}\,\Delta\bm\theta^{(1)}\big) \big/ \sum_d \Delta\bm\theta^{(d)\top}\mathbf{H}\,\Delta\bm\theta^{(d)}$, using the rank-$d$ directional curvature of \cref{eqn:fd_curv}; both concentrate over the same window, from ${\sim}70\%$ to ${\sim}85$--$95\%$. Early in training the spectrum is closer to isotropic, so \samuon{} and \samuonlite{} delay the full anisotropic shaping and warm the profile from the uniform scaling of Muon to its target over the first $30\%$ of training (much slower than standard warmup schedules).}
    \label{fig:warmup_motivation}
\end{figure}

\begin{table}[ht]
\caption{\textbf{Optimiser hyperparameters.} All values are tuned on the 124M model per batch size, with optima listed as bs $1024/2048/4096$ and the full sweeps in \cref{tab:hp_search_124m}. Both \samuon{} instantiations run at the fixed spectral radius $50$. The profile rank $k$ of \samuon{} follows the fixed width rule described below and is not tuned. All values transfer across model scales, except the AdamW learning rate, which is halved at each scale step.}
\label{tab:optim_params}
\begin{center}
\begin{small}
\setlength{\tabcolsep}{4pt}

\begin{tabular}{l|cccc}
\toprule
\textbf{Parameter} & \textbf{AdamW} & \textbf{Muon (Scion)} & \textbf{\samuonlite{}} & \textbf{\samuon{}} \\
\midrule
\multicolumn{5}{l}{\textit{Global Settings}} \\
Weight Decay & 0.0 & 0.0 & 0.0 & 0.0 \\
Gradient Clip & 1.0 & None & None & None \\
LR Schedule & Trapezoidal & Trapezoidal & Trapezoidal & Trapezoidal \\
LR Warmup & 5\% & 0\% & 0\% & 0\% \\
\midrule
\multicolumn{5}{l}{\textit{Hyperparameter sweeps at 124M (optima at bs 1024/2048/4096)}} \\
LR / radius grid & $2^{x}$, half steps & $\{35.36, \dots, 141.42\}$ & fixed at 50 & fixed at 50 \\
Optimal LR / radius & $2^{-8.5}/2^{-9}/2^{-8.5}$ & $70.71/50/70.71$ & 50 & 50 \\
Tail boost $\gamma$ grid & -- & -- & $\{3.54, \dots, 14.14\}$ & $\{3.54, \dots, 14.14\}$ \\
Optimal $\gamma$ & -- & -- & $10/7.07/10$ & $7.07$ (all) \\
Profile rank $k$ & -- & -- & 1 & $39/50/71$ \\
Profile warmup & -- & -- & 30\% (cosine) & 30\% (cosine) \\
\midrule
\multicolumn{5}{l}{\textit{Transfer to 300M and 1B}} \\
Across model scales & LR halved per step & unchanged & unchanged & $k$ by width rule \\
\bottomrule
\end{tabular}

\end{small}
\end{center}
\end{table}

\subsection{Newton--Schulz Parameters}
\label{app:ns_parameters}

All Muon-based optimisers in our experiments use the same custom $4 \times 3$ Newton--Schulz scheme: four iterations, each requiring three matrix multiplications. For a normalised input $\mathbf{X}_1$, iteration $\ell$ applies
\begin{equation}
    \mathbf{X}_{\ell+1}
    = a_\ell \mathbf{X}_\ell
    + b_\ell \bigl(\mathbf{X}_\ell\mathbf{X}_\ell^\top\bigr)\mathbf{X}_\ell
    + c_\ell \bigl(\mathbf{X}_\ell\mathbf{X}_\ell^\top\bigr)^2\mathbf{X}_\ell,
    \qquad \ell=1,\dots,4.
    \label{eq:ns_iteration}
\end{equation}
The iteration-specific coefficients in \cref{tab:ns_coefficients} are optimised to produce a stable response close to one across the operative range of normalised singular values. \Cref{fig:ns_response} shows the resulting response after four iterations; for normalised input singular values $\sigma \ge 0.02$, the maximum response error is $0.0069$.

\begin{table}[ht]
\caption{Coefficients of the custom $4 \times 3$ Newton--Schulz scheme.}
\label{tab:ns_coefficients}
\centering
\small
\begin{tabular}{c r r r}
\toprule
Iteration $\ell$ & $a_\ell$ & $b_\ell$ & $c_\ell$ \\
\midrule
1 & 5.30697775 & $-9.73226547$ & 4.52926445 \\
2 & 3.99123669 & $-4.20899105$ & 1.18235242 \\
3 & 2.66316843 & $-2.16650701$ & 0.56325209 \\
4 & 1.93040931 & $-1.31219244$ & 0.38289258 \\
\bottomrule
\end{tabular}
\end{table}

\clearpage
\begin{figure}[H]
    \centering
    \includegraphics[width=0.55\textwidth]{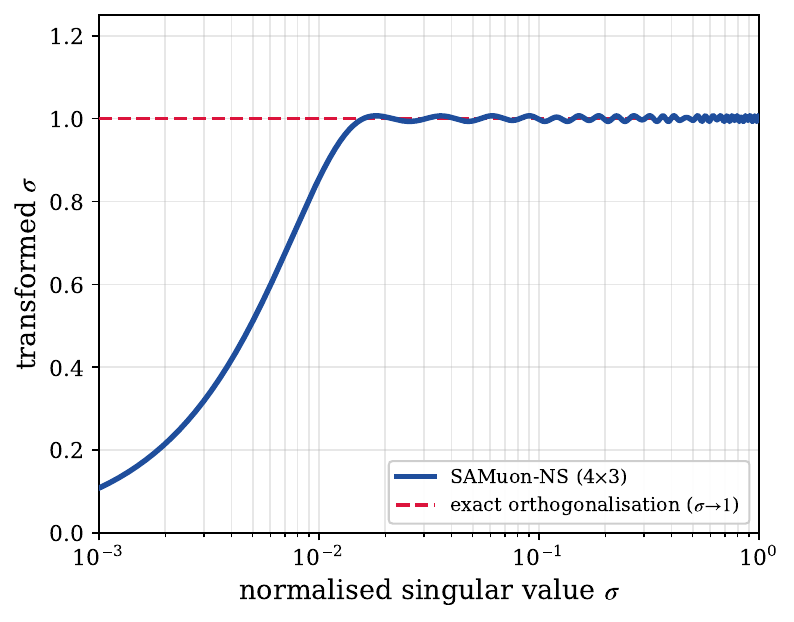}
    \caption{Response of the custom $4 \times 3$ Newton--Schulz scheme after four iterations. The dashed line denotes exact orthogonalisation, which maps every non-zero singular value to one. For normalised input singular values $\sigma \ge 0.02$, the maximum response error is $0.0069$.}
    \label{fig:ns_response}
\end{figure}

\subsection{Full sweep results}
\label{app:full_sweeps}

\cref{tab:hp_search_124m} reports all sweep results underlying the tuned hyperparameters in \cref{tab:optim_params}, covering the AdamW learning rate, the Muon spectral radius, and the tail boost of both \samuon{} instantiations. Note that \samuon{} attains a best-of-sweep loss at least as low as \samuonlite{} at all three batch sizes, with the two effectively tied at batch size 1024, and that its optimum is at $\gamma = 7.07$ in all cases whereas the \samuonlite{} optimum varies.

\begin{table}[t]
\caption{\textbf{Hyperparameter searches at 124M} (width 768, FineWeb-100B, ${\sim}$10B tokens), for each batch size. Entries are final validation loss, and the per-batch optimum of each sweep is in bold. AdamW is swept over learning rate ($2^{x}$), Muon over its spectral radius, and both \samuon{} instantiations over the tail boost $\gamma$ at the re-anchored radius $50$, the Scion default. The Muon optimum shifts with batch size ($70.71/50/70.71$), whereas the \samuon{} optimum is $\gamma = 7.07$ at all batch sizes. Searches at the superseded radius $70.71$ are omitted for the \samuon{} instantiations.}

\vspace{4pt}
\centering
\small
\setlength{\tabcolsep}{6pt}
\setlength{\aboverulesep}{0.35ex}
\setlength{\belowrulesep}{0.5ex}
\begin{tabular}{l rrr}
\toprule
Hyperparameter & bs 1024 & bs 2048 & bs 4096 \\
\midrule
\multicolumn{4}{l}{\textit{AdamW --- learning rate ($\log_2$)}} \\
\quad $-10.5$ & 3.2363 & -- & -- \\
\quad $-10$ & -- & 3.2620 & -- \\
\quad $-9.5$ & 3.2113 & -- & 3.3266 \\
\quad $-9$ & -- & \textbf{3.2410} & -- \\
\quad $-8.5$ & \textbf{3.2105} & -- & \textbf{3.3050} \\
\quad $-8$ & -- & 3.2462 & -- \\
\quad $-7.5$ & 3.2360 & -- & 3.3392 \\
\quad $-7$ & -- & 3.2963 & -- \\
\quad $-6.5$ & -- & -- & 3.4059 \\
\midrule
\multicolumn{4}{l}{\textit{Muon --- spectral radius}} \\
\quad 35.36 & 3.1625 & 3.1768 & 3.2046 \\
\quad 50 & 3.1641 & \textbf{3.1726} & 3.1979 \\
\quad 70.71 & \textbf{3.1622} & 3.1743 & \textbf{3.1953} \\
\quad 100 & 3.1641 & 3.1757 & 3.1965 \\
\quad 141.42 & 3.1649 & 3.1761 & 3.1986 \\
\midrule
\multicolumn{4}{l}{\textit{\samuon{} --- $\gamma$ at radius 50}} \\
\quad 3.54 & 3.1473 & 3.1535 & 3.1734 \\
\quad 5 & 3.1445 & 3.1502 & 3.1686 \\
\quad 7.07 & \textbf{3.1435} & \textbf{3.1484} & \textbf{3.1658} \\
\quad 10 & 3.1442 & 3.1486 & 3.1662 \\
\quad 14.14 & -- & 3.1504 & -- \\
\midrule
\multicolumn{4}{l}{\textit{\samuonlite{} --- $\gamma$ at radius 50}} \\
\quad 3.54 & 3.1462 & 3.1553 & 3.1724 \\
\quad 5 & 3.1445 & 3.1523 & 3.1699 \\
\quad 7.07 & 3.1449 & \textbf{3.1511} & 3.1680 \\
\quad 10 & \textbf{3.1438} & 3.1516 & \textbf{3.1678} \\
\quad 14.14 & -- & 3.1550 & -- \\
\bottomrule
\end{tabular}

\label{tab:hp_search_124m}
\end{table}

\subsection{\samuon{} / \samuonlite{} configurations}

Both instantiations run at the fixed spectral radius $50$ and add the tail boost $\gamma$, swept per batch size on the 124M model, with the selected optima summarised in \cref{tab:optim_params} and the complete sweeps reported in \cref{tab:hp_search_124m}. The tail boost is ramped from the Muon value over the cosine profile warmup spanning the first $30\%$ of training. Setting $\gamma = 1$ recovers Muon exactly. \samuonlite{} captures the head with a rank-1 power iteration (\cref{sec:method}). \samuon{} instead requires the leading $k$ singular pairs of the momentum buffer, computed per step with \texttt{torch.svd\_lowrank}. Its profile rank follows the width rule $k(d_\text{model}) = \left\lfloor 32\sqrt{d_\text{model}/512}\,\right\rfloor$ below, \textit{i.e.}\ $k = 39/50/71$ at the trained widths $768/1280/2560$, held fixed and not tuned.

\textbf{Low-rank estimator hyperparameters.} For the 124M, 300M, and 1B models, respectively, \samuon{} uses $k=39/50/71$, with \texttt{torch.svd\_lowrank} parameters $q=44/55/76$ (five oversampling dimensions) and $n_{\mathrm{iter}}=4/5/6$. At the same model scales, \samuonlite{} uses a rank-one power iteration with $10/12/14$ iterations.

\textbf{Choice of the profile rank $k$.} The plateau onset measured on the width-512 probing model is $k \approx 32$ (the left panel of \cref{fig:lgl_shape}); the width rule $k(d_\text{model}) = \left\lfloor 32\sqrt{d_\text{model}/512}\,\right\rfloor$ transfers this calibration across model widths, giving $k = 39/50/71$ at the trained widths $768/1280/2560$. We present the square-root rule as a calibrated heuristic rather than a derived law. However, we believe this heuristic should be reliable due to the following two reasons. First, the profile is mostly flat beyond the log-rank-linear band, so a moderately mis-set $k$ mostly re-labels plateau ranks whose target scale is close to $\gamma$ regardless. Second, the sweeps of \cref{tab:hp_search_124m} are flat around their optima, indicating that the allocation tolerates moderate mis-specification of its scales. Under this rule \samuon{} attains the best or joint-best final loss at every trained width (\cref{sec:main_results}). A power-law dependence of the spectral structure of the Muon momentum matrices on model size has been independently documented by \citet{lowrankMuonBad}; the randomised low-rank estimate itself is accurate and well-conditioned at these ranks \citep{lowrankSVD}, and LiMuon \citep{limuon} provides a precedent for running a randomised SVD of the momentum buffer inside a Muon-style optimiser. Directly measuring the plateau onset at the two largest trained widths, rather than transferring it, is left for future work; the width-768 probing of \cref{app:w768_probing} shows a modest shift of the approximate plateau onset in the direction predicted by the rule, although additional widths are needed to test the proposed scaling.

\subsection{Wall-clock cost of the optimiser update step}
\label{app:wallclock}

To quantify the per-iteration overhead discussed in \cref{subsec:complexity}, we time one full training iteration of the 1B model under its real batch-size-1024 configuration on a single NVIDIA RTX 6000 Ada GPU; this benchmark serves the wall-clock evaluation only, while the training runs of \cref{sec:experimental_setup} execute on A100 GPUs. Per iteration, the GPU performs 64 microbatch forward+backward passes. \cref{tab:wallclock} reports the GPU wall-clock breakdown. The forward+backward computation dominates at $91.2\%$ of the Muon iteration, with the Newton-Schulz whitening accounting for the remaining $8.8\%$. On top of this, the rank-1 power iteration of \samuonlite{} adds $41.2$ ms ($0.5\%$ of the Muon iteration), confirming its wall-clock parity with Muon. The rank-$k$ \texttt{torch.svd\_lowrank} call of \samuon{} adds $657.8$ ms ($7.4\%$), nearly as much as the Newton-Schulz whitening itself despite its far lower FLOP count ($O(mnk)$ against $O(mnr)$, \cref{subsec:complexity}), which reflects the unoptimised kernel rather than the arithmetic cost of the method.

\begin{table}[ht]
\caption{\textbf{Per-GPU wall-clock cost of one training iteration} (1B model, batch-size-1024 configuration), measured on a single NVIDIA RTX 6000 Ada GPU. The full Muon iteration is the forward+backward computation plus the Newton-Schulz whitening ($8916$ ms in total), and percentages are relative to this total. The \samuon{} and \samuonlite{} rows give the \emph{additional} cost of each variant's head estimate on top of the Muon iteration.}

\vspace{4pt}
\centering
\small
\setlength{\aboverulesep}{0.35ex}
\setlength{\belowrulesep}{0.5ex}
\begin{tabular}{l r r}
\toprule
Component & ms/iter & \% of Muon iter \\
\midrule
Forward + backward ($\times 64$ microbatches, compiled, bf16) & 8133.8 & 91.2\% \\
Muon: Newton-Schulz ($4 \times 3$) & 782.1 & 8.8\% \\
\midrule
\samuon{}: $+$\,\texttt{torch.svd\_lowrank} & 657.8 & 7.4\% \\
\samuonlite{}: $+$\,power iteration (bf16) & 41.2 & 0.5\% \\
\bottomrule
\end{tabular}

\label{tab:wallclock}
\end{table}

\section{Measuring the token-efficiency improvement}
\label{app:speedup}

For each model and batch-size setting, we train the tuned-Muon baseline for $N$ iterations and use
its final validation loss as the target. We then ask how short a SAMuon or SAMuon-lite run can be
while still reaching that target. Every candidate run of length $T$ uses a complete learning-rate
schedule, including the final warmdown, compressed into $T$ iterations. We therefore compare final
losses from fully annealed runs, rather than comparing the loss curves at an intermediate iteration.

If $T^\star$ is the estimated number of iterations needed to match the target, the token-efficiency
improvement is
\begin{equation}
\text{improvement} = 1 - \frac{T^\star}{N}.
\label{eq:speedup-def}
\end{equation}
For example, matching the baseline after $0.8N$ iterations gives a $20\%$ improvement.

% <<<GENERATED-TABLE (speedup_full_table.py)
\begin{table}[!t]
\caption{\textbf{Token-efficiency improvement of SAMuon over tuned Muon at matched final validation loss.} $N$ is the tuned-Muon schedule length, and $T^\star$ is the estimated SAMuon or SAMuon-lite schedule length that matches Muon's final loss, found using Algorithm~\ref{alg:speedup} and checked with a confirmation run. SAMuon uses spectral radius $50$; Muon uses its tuned radius for each batch size. The shaping warmup remains fixed at $0.30N$ absolute steps, making the reported improvements conservative. The 1B runs use 20B tokens; the others use 10B.}

\vspace{4pt}
\centering
\small
\setlength{\tabcolsep}{6pt}
\setlength{\aboverulesep}{0.35ex}
\setlength{\belowrulesep}{0.5ex}
\begin{tabular}{ll c r rr}
\toprule
 &  &  &  & \multicolumn{2}{c}{Token-effic. improv. $1-T^\star/N$} \\
\cmidrule(lr){5-6}
Model & Batch & Tokens & $N$ & SAMuon & SAMuon-lite \\
\midrule
124M (w768) & 1024 & 10B & 9600 & 20.3\% & 20.3\% \\
 & 2048 & 10B & 4800 & 23.8\% & 21.5\% \\
 & 4096 & 10B & 2400 & 24.0\% & 22.1\% \\
\midrule
300M (w1280) & 1024 & 10B & 9600 & 17.7\% & 15.9\% \\
 & 2048 & 10B & 4800 & 18.8\% & 17.1\% \\
 & 4096 & 10B & \multicolumn{3}{c}{\textit{not trained}} \\
\midrule
1B (w2560) & 1024 & 20B & 19200 & 13.3\% & 13.3\% \\
 & 2048 & 20B & \multicolumn{3}{c}{\textit{not trained}} \\
 & 4096 & 20B & 4800 & 15.6\% & 14.4\% \\
\bottomrule
\end{tabular}

\label{tab:speedup-grid}
\end{table}
% GENERATED-TABLE>>>

\subsection{Iteration matching}

Longer fully annealed runs generally achieve lower final validation loss over the range considered.
We therefore search for the run length at which the variant reaches the baseline target. We start
with candidate lengths of $0.70N$ and $0.90N$, corresponding to improvements of $30\%$ and $10\%$.
The shorter run should finish above the target loss, whereas the longer run should match or improve
upon it. If they do not bracket the target, we expand the search interval.

Within this interval, false-position interpolation uses the two endpoint losses to propose the next
run length. We execute a complete, fully annealed run at that length and replace the shorter or
longer endpoint according to whether the run misses or reaches the target. We repeat this process
until the two endpoints differ by at most $0.005N$ iterations. Finally, we estimate $T^\star$ by
linear interpolation between the runs nearest the target and execute a slightly longer confirmation
run to check that the target is reached.

\begin{algorithm}[t]
\caption{Iteration matching for the token-efficiency measurement.}
\label{alg:speedup}
\begin{algorithmic}[1]
\Require tuned-Muon schedule length $N$ and final validation loss; SAMuon variant
\State Verify that a fully annealed $N$-iteration variant run reaches the baseline loss
\State Set the shorter and longer candidate lengths to $0.70N$ and $0.90N$
\State Train and evaluate fully annealed variant runs at both candidate lengths
\State Expand the interval if necessary until the shorter run misses the target and the longer run reaches it
\While{the candidate lengths differ by more than $0.005N$}
  \State Choose the next length by false-position interpolation between the endpoint lengths and losses
  \State Train and evaluate a fully annealed variant run at the proposed length
  \If{the final loss is above the baseline target}
    \State Replace the shorter endpoint with this run
  \Else
    \State Replace the longer endpoint with this run
  \EndIf
\EndWhile
\State Estimate $T^\star$ by linear interpolation between the runs nearest the target
\State Run a schedule half the final interval width longer than $T^\star$ and confirm that it reaches the target
\State \Return $1-T^\star/N$
\end{algorithmic}
\end{algorithm}

Final validation losses vary slightly between runs, partly because the optimiser uses a randomised
low-rank SVD. The search therefore retains executed runs on both sides of the target rather than
relying on a single interpolated estimate. The additional, slightly longer run provides a
conservative check that the reported matched length reaches the target despite this variation.

The SAMuon shaping warmup is fixed at $0.30N$ absolute iterations and is not shortened with the
candidate schedule. A matched run with $T<N$ therefore spends more than $30\%$ of its iterations
in warmup (between $34.6\%$ and $39.5\%$ for the settings in \cref{tab:speedup-grid}). This
puts the shortened runs at a disadvantage, so the reported token-efficiency improvements are
conservative.

\section{Per-Module Sign Alignment}
\label{app:sign_alignment}

The main-text probe of \cref{sec:linesearch_analysis} applies each weight matrix's rank-$d$ direction in the fixed orientation $-\bm u_{j,d}\bm v_{j,d}^\top$ that a real optimiser uses, and clusters by rank across weight matrices. The resulting profile therefore averages over an inter-block \emph{sign} structure, \textit{i.e.}\ how often a weight matrix's buffer direction agrees with the out-of-sample gradient, which \cref{fig:sign_alignment} records. This agreement measures how trustworthy a direction is. A direction that agrees with the gradients of unseen batches makes transferable progress, whereas a direction that flips sign from batch to batch has a greater chance of disrupting training. The tail probes align poorly under this measure, which explains why the in-sample analysis of \cref{app:insample_probing} is far more optimistic along the tail: in-sample, each direction is evaluated on the batch whose gradient formed the buffer, where its alignment is much better. The alignment also differs by weight-matrix type. The attention projections, and the query and key projections in particular, show the worst tail alignment, so their tail directions deserve the least trust; the QK-clip mechanism used in large-scale Muon training \citep{qkclip} restrains exactly these matrices, so the probing framework predicts training behaviour observed in practice. This cross-matrix agreement further determines whether contributions add or cancel within a rank, and is one source of the mostly flat bulk.

\begin{figure}[ht]
    \centering
    \includegraphics[width=0.9\textwidth]{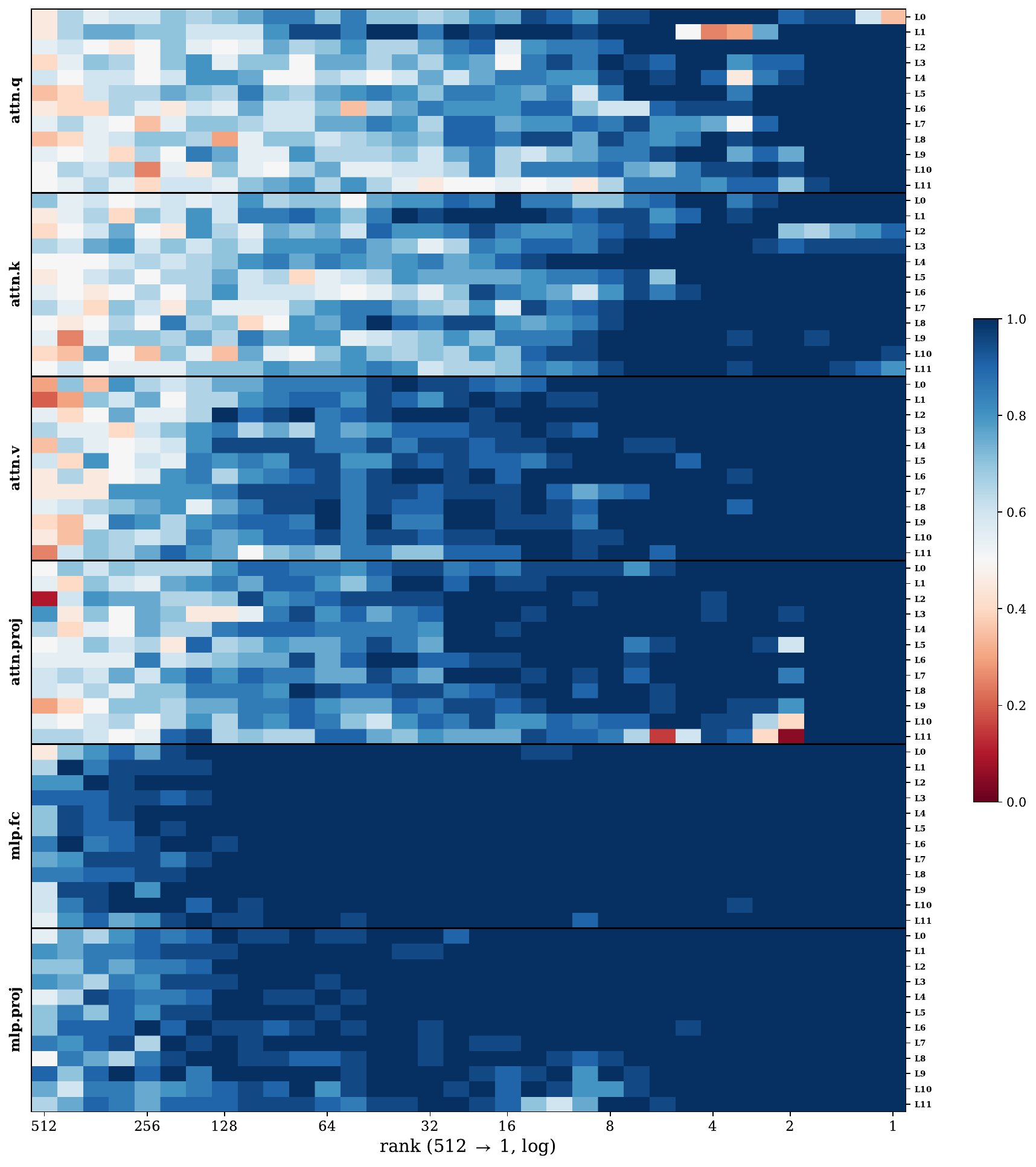}
    \caption{\textbf{Per-weight-matrix sign-alignment probability.} For the checkpoint at iteration 1000, the probability that a weight matrix's rank-$d$ buffer direction $\bm u_{j,d}\bm v_{j,d}^\top$ is aligned (has positive projection) with the out-of-sample gradient, estimated over 20 held-out batches. Rows group the six weight types (attention query, key, value, and output projections; multilayer perceptron (MLP) up- and down-projections), each shown for all twelve layers L0--L11; columns run from the deep tail (rank 512) to the head (rank 1, log axis). Blue marks near-certain alignment, red near-certain anti-alignment, and white a coin-flip. Two patterns stand out. First, the top few ranks are almost always aligned (the deep-blue right edge across all weight matrices), so the head's orientation is effectively deterministic. Second, in the tail the four attention projections drift towards $0.5$ and below, \textit{i.e.}\ their low-energy directions flip sign freely from batch to batch, whereas the MLP projections stay strongly aligned across almost the entire spectrum. The query and key projections show the worst tail alignment, so their tail directions are the least trustworthy; compare this observation with the QK-clip practice of \citet{qkclip}. This weight-matrix-dependent sign-flip structure is the inter-block variation that our per-rank clustering averages over (\cref{sec:discussion}); the cross-batch cancellation it produces in the tail is one source of the mostly flat bulk.}
    \label{fig:sign_alignment}
\end{figure}

\end{document}